\documentclass[nohyperref]{article}

\usepackage{amsmath,amsfonts,bm}

\def\eqref#1{equation~\ref{#1}}

\def\1{\bm{1}}

\DeclareMathAlphabet{\mathsfit}{\encodingdefault}{\sfdefault}{m}{sl}
\SetMathAlphabet{\mathsfit}{bold}{\encodingdefault}{\sfdefault}{bx}{n}

\usepackage{microtype}
\usepackage{graphicx}
\usepackage{subfigure}
\usepackage{booktabs} 

\usepackage{hyperref}

\usepackage[accepted]{icml2022}

\usepackage{amsmath}
\usepackage{amssymb}
\usepackage{mathtools}
\usepackage{amsthm}

\usepackage[capitalize,noabbrev]{cleveref}

\usepackage{placeins}

\usepackage{url}            
\usepackage{amsfonts}       
\usepackage{nicefrac}       
\usepackage{xcolor}         
\usepackage{subfigure}
\usepackage{multirow}
\usepackage{mathabx}
\usepackage{enumerate}
\usepackage{pifont}
\usepackage{threeparttable}

\newtheorem{claim}{Claim}
\newcommand{\yes}{\ding{51}}

\theoremstyle{plain}
\newtheorem{theorem}{Theorem}[section]
\newtheorem{proposition}[theorem]{Proposition}

\theoremstyle{definition}

\theoremstyle{remark}

\usepackage[textsize=tiny]{todonotes}

\icmltitlerunning{A New Perspective on the Effects of Spectrum in Graph Neural Networks}

\begin{document}
	
	\twocolumn[
	\icmltitle{A New Perspective on the Effects of Spectrum in Graph Neural Networks}
	
	
	
	
	\begin{icmlauthorlist}
		\icmlauthor{Mingqi Yang}{dlut}
		\icmlauthor{Yanming Shen}{dlut}
		\icmlauthor{Rui Li}{dlut}
		\icmlauthor{Heng Qi}{dlut}
		\icmlauthor{Qiang Zhang}{dlut}
		\icmlauthor{Baocai Yin}{dlut,pcl}
	\end{icmlauthorlist}
	
	\icmlaffiliation{dlut}{Dalian University of Technology, China}
	\icmlaffiliation{pcl}{Peng Cheng Laboratory, China}
	
	\icmlcorrespondingauthor{Yanming Shen}{shen@dlut.edu.cn}
	
	\icmlkeywords{Machine Learning, ICML}
	
	\vskip 0.3in
	]
	
	
	
	\printAffiliationsAndNotice{}  

\begin{abstract}
	Many improvements on GNNs can be deemed as operations on the spectrum of the underlying graph matrix, which motivates us to directly study the characteristics of the spectrum and their effects on GNN performance. By generalizing most existing GNN architectures, we show that the correlation issue caused by the $unsmooth$ spectrum becomes the obstacle to leveraging more powerful graph filters as well as developing deep architectures, which therefore restricts GNNs' performance. Inspired by this, we propose the correlation-free architecture which naturally removes the correlation issue among different channels, making it possible to utilize more sophisticated filters within each channel. The final correlation-free architecture with more powerful filters consistently boosts the performance of learning graph representations. Code is available at \url{https://github.com/qslim/gnn-spectrum}.
\end{abstract}

\section{Introduction}
Although graph neural network (GNN) communities are in a rapid development of both theories and applications, there is still a lack of a generalized understanding of the effects of the graph's spectrum in GNNs. As we can see, many improvements can finally be unified into different operations on the spectrum of the underlying graph, while their effectiveness is interpreted by several well-accepted isolated concepts: \cite{pmlr-v97-wu19e,zhu2021interpreting,klicpera_predict_2019,klicpera2019diffusion,chien2021adaptive,balcilar2021analyzing} explain it in the perspective of simulating low/high pass filters; \cite{chenWHDL2020gcnii,xu2018representation,liu2020towards,li2018deeper} interpret it as ways of alleviating oversmoothing phenomenon in deep architectures; \cite{cai2020graphnorm} adopts the conception of normalization operation in neural networks and applies it to graph data.
Since these improvements all indirectly operate on the spectrum,
it motivates us to study the potential connections between the GNN performance and the characteristics of the graph's spectrum. If we can find such a connection, it would provide a deeper and generalized insight into these seemingly unrelated improvements associated with the graph's spectrum (low/high pass filter, oversmoothing, graph normalization, etc), and further identify potential issues in existing architectures.
To this end, we first consider the simple correlation metric: cosine similarity among signals, and study the relations between it and the graph's spectrum in the graph convolution operation. It provides a new perspective that in existing GNN architectures, the distribution of eigenvalues of the underlying graph matrix controls the cosine similarity among signals. An ill-posed $unsmooth$ spectrum would easily make signals over-correlated which is evidence of information loss.

Compared with oversmoothing studies~\cite{li2018deeper,oono2020graph,rong2019dropedge,huang2020tackling}, the correlation analysis associated with the graph's spectrum further indicates that the correlation issue is essentially caused by the graph's spectrum. In other words, for graph topologies with an unsmooth spectrum, the issue can appear even with a shallow architecture, and a deep model further makes the spectrum less smooth and eventually exacerbates this issue. Meanwhile, the correlation analysis also provides a unified interpretation of the effectiveness of various existing improvements associated with the graph's spectrum since they all implicitly impose some constraints on the spectrum to alleviate the correlation issue. However, these improvements are trade-offs between alleviating the correlation issue and applying more powerful graph filters: since a filter implementation directly reflects on the spectrum, a more appropriate filter for relevant signal patterns may correspond to an ill-posed spectrum, which in return will not gain performance improvements. Hence, in general GNN architectures, the correlation issue becomes the obstacle to applying more powerful filters.
As we can see, although one can approximate more sophisticated graph filters by increasing the order $k$ of the polynomial theoretically~\cite{shuman2013emerging}, in the popular models, simple filters, e.g. low-pass filter~\cite{kipf2017semi,pmlr-v97-wu19e}, or the fixed filter coefficients~\cite{klicpera_predict_2019,klicpera2019diffusion} serve as the practical applicable choice.

With all the above understandings, the key solution is to decouple the correlation issue from the filter design, which results in our correlation-free architecture. In contrast to existing approaches, it allows to focus on exploring more sophisticated filters without the concern of the correlation issue. With this guarantee, we can improve the approximation abilities of polynomial filters to better approximate the desired more complex filters~\cite{hammond2011wavelets,defferrard2016convolutional}.
However, we also find that it cannot be achieved by simply increasing the number of polynomial bases as the basis characteristics implicitly restrict the number of available bases in the resulting polynomial filter. For this reason, commonly used (normalized) adjacency or Laplacian matrix where its spectrum serves as the basis cannot effectively utilize high-order bases. To address this issue, we propose new graph matrix representations, which are capable of leveraging more bases and learnable filter coefficients to better respond to more complex signal patterns. The resulting model significantly boosts performance on learning graph representations. Although there are extensive studies on the polynomial filters including the fixed coefficients and learnable coefficients~\cite{defferrard2016convolutional,8521593,chien2021adaptive,he2021bernnet}, to the best of our knowledge, they all focus on the coefficients design and use the (normalized) adjacency or Laplacian matrix as a basis. Therefore, our work is well distinguished from them.
Our contributions are summarized as follows:
\begin{itemize}
	\item 
	\vspace{-5pt}
	We show that general GNN architectures suffer from the correlation issue and also quantify this issue with spectral smoothness;
	\item
	\vspace{-5pt}
	We propose the correlation-free architecture that decouples the correlation issue from graph convolution;
	\item
	\vspace{-5pt}
	We show that the spectral characteristics also hinder the approximation abilities of polynomial filters and address it by altering the graph's spectrum.
\end{itemize}

\section{Preliminaries}

Let $\mathcal G=(\mathcal V, \mathcal E)$ be an undirected graph with node set $\mathcal V$ and edge set $\mathcal E$.
We denote $n = |\mathcal V|$ the number of nodes, $A\in\mathbb A^{n\times n}$ the adjacency matrix and $H\in\mathbb R^{n\times d}$ the node feature matrix where $d$ is the feature dimensionality.
$\bm h\in\mathbb R^n$ is a graph signal that corresponds to one dimension of $H$.

\textbf{Spectral Graph Convolution~\cite{hammond2011wavelets,defferrard2016convolutional}.}
The definition of spectral graph convolution relies on Fourier transform on the graph domain.
For a signal $\bm h$ and graph Laplacian $L=U\Lambda U^{\top}$, we have Fourier transform $\hat x=U^{\top}x$ and inverse transform $x=U\hat x$.
Then, the graph convolution of a signal $\bm h$ with a filter $\bm g_{\theta}$ is
\begin{equation}
	\label{equ:graph_conv}
	\begin{aligned}
		\bm g_{\theta}*\bm h=U\left(\left(U^{\top}\bm g_{\theta}\right)\odot\left(U^{\top}\bm h\right)\right)=U\hat{G}_{\theta^{\prime}}U^{\top}\bm h,
	\end{aligned}
\end{equation}
where $\hat{G}_{\theta^{\prime}}$ denotes a diagonal matrix in which the diagonal corresponds to spectral filter coefficients.
To avoid eigendecomposition and ensure scalability, $\hat{G}_{\theta^{\prime}}$ is approximated by a truncated expansion in terms of Chebyshev polynomials $T_k(\tilde{\Lambda})$ up to the $k$-th order~\cite{hammond2011wavelets}, which is also the polynomials of $\Lambda$,
\begin{equation}
	\label{equ:poly_filter}
	\begin{aligned}
		\hat{G}_{\theta^{\prime}}(\Lambda)\approx\sum_{i=0}^{k} \theta_i^{\prime} T_i(\tilde{\Lambda})=\sum_{i=0}^{k}\theta_{i}\Lambda^{i},
	\end{aligned}
\end{equation}
where $\tilde{\Lambda}=\frac{2}{\lambda_{\max}}\Lambda-I_n$.
Now the convolution in Eq.~\ref{equ:graph_conv} is
\begin{equation}
	\label{equ:poly_conv}
	\begin{aligned}
		U\hat{G}_{\theta^{\prime}}U^{\top}\bm h\approx U\left(\sum_{i=0}^{k}\theta_{i}\Lambda^{i}\right)U^{\top}\bm h=\sum_{i=0}^{k}\theta_{i}L^{i}\bm h.
	\end{aligned}
\end{equation}
Note that this expression is $k$-localized since it is a $k$-order polynomial in the Laplacian, i.e., it depends only on nodes that are at most $k$ hops away from the central node.

\textbf{Graph Convolutional Network (GCN)~\cite{kipf2017semi}.}
GCN is derived from $1$-order Chebyshev polynomials with several approximations.
The authors further introduce the renormalization trick $\tilde{D}^{-\frac{1}{2}}\tilde{A}\tilde{D}^{-\frac{1}{2}}$ with $\tilde{A}=A+I_n$ and $\tilde{D}_{i i}=\sum_{j}\tilde{A}_{i j}$.
Also, GCN can be generalized to multiple input channels and a layer-wise model:
\begin{equation}
	\label{equ:gcn_conv}
	\begin{aligned}
		H^{(l+1)}=\sigma\left(\tilde{D}^{-\frac{1}{2}}\tilde{A}\tilde{D}^{-\frac{1}{2}}H^{(l)}W^{(l)}\right),
	\end{aligned}
\end{equation}
where $W$ is learnable matrix and $\sigma$ is nonlinear function.

\textbf{Graph Diffusion Convolution (GDC)~\cite{klicpera2019diffusion}.}
A generalized graph diffusion is given by the diffusion matrix:
\begin{equation}
	\label{equ:gdc_conv}
	\begin{aligned}
		H=\sum_{k=0}^{\infty} \theta_k T^k,
	\end{aligned}
\end{equation}
with the weight coefficients $\theta_k$ and the generalized transition matrix $T$.
$T$ can be $T_{rw}=AD^{-1}$, $T_{sym}=D^{-\frac{1}{2}}AD^{-\frac{1}{2}}$ or others as long as they are convergent.
GDC can be viewed as a generalization of the original definition of spectral graph convolution, which also applies polynomial filters but not necessarily the Laplacian.

\section{Revisiting Existing GNN Architectures}

\begin{table*}[t]
	\small
	\centering
	\caption{A summary of $p_{\gamma}$ in Eq.~\ref{equ:gnn_generalization} in general graph convolutions.}
	\label{tab:gnn_summary}
	\vspace{5pt}
	\resizebox{1.0\textwidth}{!}{%
		\begin{tabular}{c|cccccccccc}
			\toprule
			&GCN  &SGC  &APPNP  &GCNII  &GDC  &SSGC  &GPR  &ChebyNet  &CayleNet  &BernNet   	         \\
			\midrule
			Poly-basis			&General &General &Residual &Residual &General &General &General &Chebyshev & Cayle &Bernstein    \\
			\midrule
			Poly-coefficient	&Fixed &Fixed &Fixed &Fixed &Fixed &Fixed &Learnable &Fixed &Learnable &Learnable     \\
			\bottomrule
		\end{tabular}
	}
	\vspace{-10pt}
\end{table*}
We first generalize existing spectral graph convolution as follows
\begin{equation}
	\label{equ:gnn_generalization}
	\begin{aligned}
		H=\sigma\bigl(p_{\gamma}(S)f_{\Theta}(H)\bigr),
	\end{aligned}
\end{equation}
where $S$ is the graph matrix, e.g. adjacency or Laplacian matrix and their normalized forms.
$p_{\gamma}:\mathbb R^{n\times n}\rightarrow\mathbb R^{n\times n}$ is the polynomial of graph matrices with coefficients $\gamma\in\mathbb R^k$ for a $k$-order polynomial.
$f_{\Theta}:\mathbb R^d\rightarrow\mathbb R^{d^{\prime}}$ is the feature transformation neural network with the learnable parameters $\Theta$.
In SGC~\cite{pmlr-v97-wu19e}, GDC~\cite{klicpera2019diffusion}, SSGC~\cite{zhu2020simple}, and GPR~\cite{chien2021adaptive}, $p_{\gamma}$ is implemented as the general polynomial, i.e. $p_{\gamma}(S)=\sum_{i=0}^{k}\gamma_i S^i$.
Their differences are identified by the coefficients $\gamma$.
For example, SGC corresponds to a very simple form with $\gamma_i=0, i<k$ and $\gamma_k=1$.
By removing the nonlinear layer in GCNII~\cite{chenWHDL2020gcnii}, APPNP~\cite{klicpera_predict_2019} and GCNII share the similar graph convolution layer as
\begin{equation}
	\nonumber
	\begin{aligned}
		H^{(l)}=(1-\alpha)SH^{(l-1)}+\alpha Z, H^{(0)}=Z, Z=f_{\Theta}(X),
	\end{aligned}
\end{equation}
where $\alpha\in(0, 1)$ and $X\in\mathbb R^{n\times d}$ is the input node features.
By deriving its closed-form, we reformulate it with Eq.~\ref{equ:gnn_generalization} as $p_{\gamma}(S)=\sum_{i=0}^{k-1}\alpha(1-\alpha)^i S^i+(1-\alpha)^k S^k$.
In ChebyNet~\cite{defferrard2016convolutional}, CayleNet~\cite{8521593} and BernNet~\cite{he2021bernnet}, $p_{\gamma}$ corresponds to Chebyshev, Cayle and Bernstein polynomials respectively.
GPR, CayleNet and BernNet apply learnable coefficient $\gamma$, where $\gamma$ is learned as the coefficients of general, Cayle and Bernstein basis respectively.
Therefore, with our formulation in Eq.~\ref{equ:gnn_generalization}, general graph convolutions are mainly different from $p_{\gamma}$ as summarized in Tab.~\ref{tab:gnn_summary}\footnote{Here, we follow the naming convention in GCNII called initial residual connection. GCN and GCNII interlace nonlinear computations over layers, making them difficult to reformulate all layers with Eq.~\ref{equ:gnn_generalization}. But one can represent them with the recursive form as $H^{(l)}=\sigma\bigl(p_{\gamma}(S)f_{\Theta}(H^{(l-1)})\bigr)$. For example, in GCN, we have $p_{\gamma}(S)=S$ and $f_{\Theta}(H^{(l-1)})=H^{(l-1)}\Theta$ with $S=\tilde{D}^{-\frac{1}{2}}\tilde{A}\tilde{D}^{-\frac{1}{2}}$.}.

\subsection{Correlation Analysis in the Lens of Graph's Spectrum}
\label{sec:correlation_analysis}

Based on the generalized formulation of Eq.~\ref{equ:gnn_generalization}, we conduct correlation analysis on existing graph convolution in the perspective of the graph's spectrum.
We denote $\mathcal S=p_{\gamma}(S)$  for simplicity.
$\bm h\in\mathbb R^n$ denotes one channel in $f_{\Theta}(H)$.
Then the convolution on $\bm h$ is represented as $\mathcal S\bm h$.
The cosine similarity between $\bm h$ and the $i$-th eigenvector $\bm p_i$ of $\mathcal S$ is
\begin{equation}
	\label{equ:cos_signal}
	\begin{aligned}
		\cos\bigl(\langle \bm h, \bm p_i\rangle\bigr)
		=\frac{\bm h^{\top}\bm p_i}
		{\sqrt{\sum^n_{j=1}\bigl(\bm h^{\top}\bm p_j\bigr)^2}}
		=\frac{\alpha_i}{\sqrt{\sum^n_{j=1}\alpha_j^2}}.
	\end{aligned}
\end{equation}
$\alpha_i=\bm h^{\top}\bm p_i$ is the weight of $\bm h$ on $\bm p_i$ when representing $\bm h$ with the set of orthonormal bases $\bm p_i,i\in[n]$.
The cosine similarity between $\mathcal S\bm h$ and $\bm p_i$ is
\begin{equation}
	\label{equ:cos_signal_eigenvalue}
	\begin{aligned}
		\cos\bigl(\langle \mathcal S\bm h, \bm p_i\rangle\bigr)
		=\frac{\alpha_i\lambda_i}{\sqrt{\sum^n_{j=1}\alpha_j^2\lambda_j^2}}.
	\end{aligned}
\end{equation}
The detailed derivations of Eq.~\ref{equ:cos_signal} and Eq.~\ref{equ:cos_signal_eigenvalue} are given in Appendix~\ref{deriv:equ:cos_signal_eigenvalue}.

Eq.~\ref{equ:cos_signal_eigenvalue} builds the connection between the cosine similarity and the spectrum of the underlying graph matrix.
We say the spectrum is $smooth$ if all eigenvalues have similar magnitudes.
By comparing Eq.~\ref{equ:cos_signal} and Eq.~\ref{equ:cos_signal_eigenvalue}, it shows that the graph convolution operation with the unsmooth spectrum, i.e., dissimilar eigenvalues, results in signals correlated (a higher cosine similarity) to the eigenvectors corresponding to larger magnitude eigenvalues and orthogonal (a lower cosine similarity) to the eigenvectors corresponding to smaller magnitude eigenvalues.
In the case where 0 eigenvalue is involved in the spectrum, signals would lose information in the direction of the corresponding eigenvectors.
In the deep architecture, this problem would further be exacerbated:
\begin{proposition}
	\label{prop:cos_convergence}
	Assume $ \mathcal S\in \mathbb R^{n\times n}$ is a symmetric matrix with real-valued entries. $|\lambda_1|\geq|\lambda_2|\geq,\dots,\geq|\lambda_n|$ are $n$ real eigenvalues, and $\bm p_i\in\mathbb R^n,i\in[n]$ are corresponding eigenvectors. Then, for any given $\bm h,\bm h^{\prime}\in\mathbb R^n$, we have\\
	(i) $|\cos(\langle \mathcal S^{k+1}\bm h, \bm p_1\rangle)|\geq|\cos(\langle \mathcal S^k\bm h, \bm p_1\rangle)|$ and $|\cos(\langle \mathcal S^{k+1}\bm h, \bm p_n\rangle)|\leq|\cos(\langle \mathcal S^k\bm h, \bm p_n\rangle)|$ for $k=0,1,2,\dots,+\infty$;\\
	(ii) If $|\lambda_1|>|\lambda_2|$, $\lim\limits_{k\rightarrow\infty}|\cos(\langle \mathcal S^k\bm h, \bm p_1\rangle)|=\lim\limits_{k\rightarrow\infty}|\cos(\langle \mathcal S^k\bm h, \mathcal S^k\bm h^{\prime}\rangle)|=1$, and the convergence speed is decided by $|\frac{\lambda_2}{\lambda_1}|$.
\end{proposition}
We prove Proposition~\ref{prop:cos_convergence} in Appendix~\ref{proof:prop:cos_convergence}.
Proposition~\ref{prop:cos_convergence} shows that a deeper architecture violates the spectrum's smoothness, which therefore makes the input signals more correlated to each other.~\footnote{
	Here, nonlinearity is not involved in the propagation step.
	This meets the case of the decoupling structure where a multi-layer GNN is split into independent propagation and prediction steps~\cite{liu2020towards,pmlr-v97-wu19e,klicpera_predict_2019,zhu2020simple,zhang2021litegem}.
	The propagation involving nonlinearity remains unexplored due to its high complexity, except for one case of ReLU as nonlinearity~\cite{oono2020graph}.
	Most convergence analyses (such as over-smoothing) only study the simplified linear case~\cite{cai2020graphnorm,liu2020towards,pmlr-v97-wu19e,klicpera_predict_2019,zhao2020pairnorm,xu2018representation,chenWHDL2020gcnii,zhu2020simple,klicpera2019diffusion,chien2021adaptive}.
}
Finally, $\textrm{Rank}((\bm h_1, \bm h_2, \dots, \bm h_d))=1$, and the information within signals would be washed out.
Note that all the above analysis does not impose any constraint to the underlying graph such as connectivity.

\textbf{Revisiting oversmoothing via the lens of correlation issue.}
In the well-known oversmoothing analysis, the convergence is considered as $\lim _{k\rightarrow\infty} \tilde A_{\mathrm{sym}}^k H^{(0)}=H^{(\infty)}$ where each row of $H^{(\infty)}$ only depends on the
degree of the corresponding node, provided that the graph is $irreducible$ and $aperiodic$~\cite{xu2018representation,liu2020towards,zhao2020pairnorm,chien2021adaptive}.
Our analysis generalizes this result.
In our analysis, the convergence of the cosine similarity among signals does not limit a graph to be $connected$ or $normalized$ that is required in the oversmoothing analysis analogical to the stationary distribution of the Markov chain, and even does not require a model to be necessarily $deep$ :
it is essentially caused by the bad distributions of eigenvalues, while the deep architecture exacerbates it.
Interestingly, inspired by this perspective, the correlation problem actually relates to the specific topologies since different topologies correspond to different spectrum.
There exists topologies inherently with bad distributions of eigenvalues, and they will suffer from the problem even with a shallow architecture.
Also, by taking the symmetry into consideration, Proposition~\ref{prop:cos_convergence}(i) shows that the convergence of cosine similarity with respect to $k$ is also $monotonous$.
In contrast that existing results only discuss the theoretical infinite depth case, this provides more concrete evidence in the practical finite depth case that a deeper architecture can be more harmful than a shallow one.

\textbf{Revisiting graph filters via the lens of correlation issue.}
The graph filter is approximated by a polynomial in the theory of spectral graph convolution~\cite{hammond2011wavelets,defferrard2016convolutional}. Although theoretically, one can approximate any desired graph filter by increasing the order $k$ of the polynomial~\cite{shuman2013emerging}, most GNNs cannot gain improvements by enlarging $k$. Instead, the simple low-pass filter studied by many improvements on spectral graph convolution acts as the practical effective choice~\cite{shuman2013emerging,pmlr-v97-wu19e,1905.09550,muhammet2020spectral,klicpera2019diffusion}. Although there are studies involving high-pass filters to better process high-frequency signals recently, the low-pass is always required in graph convolution~\cite{zhu2020simple,zhu2021interpreting,balcilar2021analyzing,fagcn2021,gao2021message}. This can be explained in the perspective of correlation analysis. As we have shown, the graph convolution is sensitive to the spectrum. A more proper filter to better respond to relevant signal patterns may result in an unsmooth spectrum, making different channels correlated to each other after convolution. In contrast, although a low-pass filter has limited expressiveness, it corresponds to a smoother spectrum, which alleviates the correlation issue.

\section{Correlation-free Architecture}
\label{sec:channel_wise_architecture}

The correlation analysis via the lens of graph's spectrum shows that in general GNN architectures, the $unsmooth$ spectrum leads to correlation issue and therefore acts as the obstacle to developing deep architectures as well as leveraging more expressive graph filters.
To overcome this issue, a natural idea is to assign the graph convolution in different channels of $f_{\Theta}(H)$ with different spectrums, which can be viewed as a generalization of Eq.~\ref{equ:gnn_generalization} as follows
\begin{equation}
	\label{equ:channel_wise_graph_conv}
	\begin{aligned}
		H=f_{\Psi}\big(\bigl[p_{\Gamma_1}(S)f_{\Theta_1}(H), \dots , p_{\Gamma_{d^{\prime}}}(S)f_{\Theta_{d^{\prime}}}(H)\bigr]\bigr).
	\end{aligned}
\end{equation}
Both $f_{\Theta}:\mathbb R^d\rightarrow\mathbb R^{d^{\prime}}$ and $f_{\Psi}:\mathbb R^{d^{\prime}}\rightarrow\mathbb R^{d^{\prime\prime}}$ are the feature transformation neural networks with the learnable parameters $\Theta$ and $\Psi$ respectively.
$p_{\Gamma_i}$ is the $i$-th polynomial with the learnable coefficients $\Gamma_i\in\mathbb R^k$.
$f_{\Theta_i}(H)\in\mathbb R^n$ is the $i$-th channel of $f_{\Theta}(H)\in\mathbb R^{n\times d^{\prime}}$.
We denote $\bm h_i=f_{\Theta_i}(H)$ for simplicity.
Then the convolution operation on $\bm h_i$ in Eq.~\ref{equ:channel_wise_graph_conv} is
\begin{equation}
	\label{equ:channel_wise_filter}
	\begin{aligned}
		p_{\Gamma_i}(S)\bm h_i = \sum_{j=0}^{k}\Gamma_{i, j} S^j \bm h_i = U\sum_{j=0}^{k}\Gamma_{i, j}\Lambda^j U^{\top}\bm h_i
	\end{aligned}
\end{equation}
with the filter $\mathrm{diag}(g_{\Gamma_i}) = \sum_{j=0}^{k}\Gamma_{i, j}\Lambda^j$.
We denote $\bm\lambda=\bigl(\lambda_1, \lambda_2, \dots, \lambda_n\bigr)^{\top}\in\mathbb R^n$.
Then,
\begin{equation}
	\label{equ:vandermonde}
	\begin{aligned}
		g_{\Gamma_i}=\sum_{j=0}^k\Gamma_{i, j}\bm\lambda^j=
		\begin{pmatrix}
			\bm\lambda^1,\bm\lambda^2,\dots,\bm\lambda^k
		\end{pmatrix}
		\times\Gamma_i=V\times\Gamma_i,
	\end{aligned}
\end{equation}
where $V=\bigl(\bm\lambda^1,\bm\lambda^2,\dots,\bm\lambda^k\bigr)\in\mathbb R^{n\times k}$.
If $\lambda_i\neq\lambda_j$ for any $i\neq j$, i.e., the algebraic multiplicity of all eigenvalues is 1, $V$ is a Vandermonde matrix with $\textrm{Rank}(V)=\textrm{min}(n, k)$.
$V_j=\bm\lambda^j, j\in[k]$ serve as a set of $k$ bases, where each filter $g_{\Gamma_i}$ is a linear combination of $V_j$.
Hence, a larger $k$ helps to better approximate the desired filter.
When $k=n$, $V$ is a full-rank matrix and $g_{\Gamma_i}$ is sufficient to represent any desired filter with proper assignments of $\Gamma_i$.
Note that $n$ is much smaller in real-world graph-level tasks than that in node-level tasks, making $k=n$ more tractable.

By considering the columns of a Vandermonde matrix, i.e. $\bm\lambda^j, j\in[k]$ as bases, we can see that when increasing $k$ (aka applying more bases), $\lambda_i^k$ with $|\lambda_i|<<1$ goes diminishing and $\lambda_i^k$ with $|\lambda_i|>>1$ goes divergent.
To balance the diminishing and divergence problems when applying a larger $k$, we need to carefully control the range of the spectrum close to $1$ or $-1$.
General approaches have $\bm\lambda\in[0, 1]^n$~\footnote{General approaches use the (symmetry) normalized $A$, i.e. $\tilde D^{-\frac{1}{2}}\tilde A\tilde D^{-\frac{1}{2}}$, $\tilde D^{-1}\tilde A$ to guarantee its spectrum is bounded by $[-1, 1]$~\cite{kipf2017semi,klicpera2019diffusion} or the (symmetry) normalized $L$, i.e. $I-\tilde D^{-\frac{1}{2}}\tilde A\tilde D^{-\frac{1}{2}}$ to ensure the boundary [0, 2] and then rescale it to $[0, 1]$~\cite{he2021bernnet}.}.
Although there is no concern of divergence problems, $\lambda_i^k$, especially for a small $\lambda_i$, inclines to 0 when increasing $k$,
making the higher-order basis ineffective in the practical limited precision condition.

On the other hand, general approaches are less likely to learn the coefficients of polynomial filters in a completely free manner~\cite{klicpera2019diffusion,he2021bernnet}.
The specially designed coefficients to explicit modify spectrum, i.e. Personalized PageRank (PPR), heat kernel~\cite{klicpera2019diffusion}, etc or the coefficients learned under the constrained condition, i.e. Chebyshev~\cite{defferrard2016convolutional}, Cayley~\cite{8521593}, Bernstein~\cite{he2021bernnet} polynomial, etc act as the practical applicable filters.
This is probably because the polynomial filter relies on sophisticated coefficients to maintain spectral properties.
Learning them from scratch would easily fall into an ill-posed filter~\cite{he2021bernnet}.
However, by modifying the filter bases, it would relax the requirement on the coefficients, making it more suitable for learning coefficients from scratch.

Finally, although the new architecture in Eq.~\ref{equ:channel_wise_graph_conv} decouples the correlation issue from developing more powerful filters, general filter bases are less qualified for approximating more complex filters. Hence, we still need to explore more effective filter bases to replace existing ones. To this end, we will introduce two different improvements on filter bases in the following sections whose effectiveness will serve as a verification of our analysis.

\subsection{Spectral Optimization on Filter Basis}
\label{sec:basis_optim}

One can directly apply a smoothing function on the spectrum of $S$, which helps to narrow the range of eigenvalues close to 1 or -1.
There can be various approaches to this end, and in this paper, we propose the following eigendecomposition-based method for a symmetric matrix $S=P\Lambda P^{\top}$~\footnote{Although the computation of $S_{\rho}$ requires eigendecomposition, $S$ is always a symmetric matrix and the eigendecomposition on it is much faster than a general matrix.}
\begin{equation}
	\label{equ:basis_optim}
	\begin{aligned}
		S_{\rho}=P\mathrm{diag}(f_{\rho}(\lambda_i))P^{\top},
		f_{\rho}(\lambda)=
		\left\{
		\begin{array}{ll}
			-(-\lambda)^{\rho}, \lambda<0 \\
			\lambda^{\rho}, \lambda\geq0,
		\end{array}
		\right.
	\end{aligned}
\end{equation}
where $i\in[n]$.
$\rho\in(0, 1)$, $\lambda^{\rho}=e^{\rho\ln\lambda}$.
$S_{\rho}$ serves as the polynomial bases in Eq.~\ref{equ:channel_wise_filter}.
Unlike general spectral approaches, $S$ is not required to be a bounded spectrum.
It can leverage more bases while alleviating both the diminishing and divergence problems by controlling $\rho\cdot k$ in a small range.
Therefore, $S_{\rho}$ can be considered as a basis-augmentation technique as shown in Fig.~\ref{fig:basis_optim}.
\begin{figure}[ht]
	\centering
	\includegraphics[width=0.95\columnwidth]{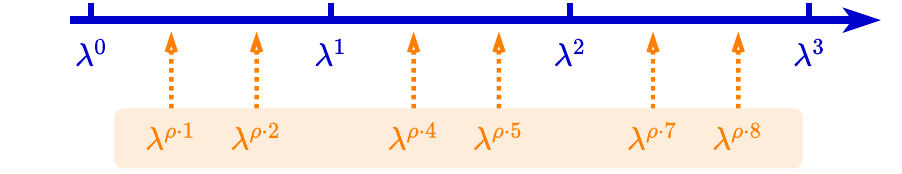}
	\vspace{-5pt}
	\caption{Assume $\lambda>0$ and $\rho=0.\dot 3$.}
	\label{fig:basis_optim}
	\vspace{-10pt}
\end{figure}

There can be other transformations on the spectrum, e.g., $P\bigl(\mathrm{Sigmoid}\bigl(\Lambda\bigr)+\rho\bigr)P^{\top}$, which have a similar effect to $S_{\rho}$.
Note that the injectivity of $f_{\rho}$ also influences the approximation ability, which is discussed in more details in Appendix~\ref{more_filter_basis}.

\subsection{Generalized Normalization on Filter Basis}

Eq.~\ref{equ:basis_optim} directly operates on the spectrum, which can achieve an accurate control on the range of the spectrum but requires eigendecomposition.
To avoid eigendecomposition, we alternatively study the effects of graph normalization on the spectrum.
We generalize the normalized adjacency matrix as follows
\begin{equation}
	\label{equ:basis_norm}
	\begin{aligned}
		\tilde{D}^{\epsilon}\tilde{A}\tilde{D}^{\epsilon}=(D+\eta I)^{\epsilon}(A+\eta I)(D+\eta I)^{\epsilon},
	\end{aligned}
\end{equation}
where $\epsilon\in[-0.5, 0]$ is the normalization coefficient and $\eta\in[0, 1]$ is the shift coefficient.
Widely-used $\tilde D^{-\frac{1}{2}}\tilde A\tilde D^{-\frac{1}{2}}$ corresponds to $\epsilon=-0.5$ and $\eta=1$.

\begin{proposition}
	\label{prop:basis_norm}
	Let $\lambda_1\geq\lambda_2\geq\dots\geq\lambda_n$ be the spectrum of $A$ and $\mu_1\geq\mu_2\geq\dots\geq\mu_n$ be the spectrum of $(D+\eta I)^{\epsilon}(A+\eta I)(D+\eta I)^{\epsilon}$, then for any $i\in[n]$, we have
	\begin{equation}
		\nonumber
		\begin{aligned}
			(\lambda_i+\eta)(d_{\mathrm{max}}+\eta)^{2\epsilon}\leq\mu_i\leq(\lambda_i+\eta)(d_{\mathrm{min}}+\eta)^{2\epsilon}
		\end{aligned},
	\end{equation}
	where $d_{\mathrm{min}}$ and $d_{\mathrm{max}}$ are the minimum and maximum degrees of nodes in the graph.
\end{proposition}
We prove Proposition~\ref{prop:basis_norm} in Appendix~\ref{proof:prop:basis_norm}.
Proposition~\ref{prop:basis_norm} extends the results in \cite{4389477}, showing that the normalization has a scaling effect on the spectrum: a smaller $\epsilon$ is likely to lead to a smaller $\mu_i$,
while a larger $\epsilon$ is likely to lead to a larger $\mu_i$.
When $\epsilon=0$, the upper and lower bounds coincide with $\mu_i=\lambda_i+\eta$.

\begin{figure*}[th]
	\centering
	\includegraphics[width=\textwidth]{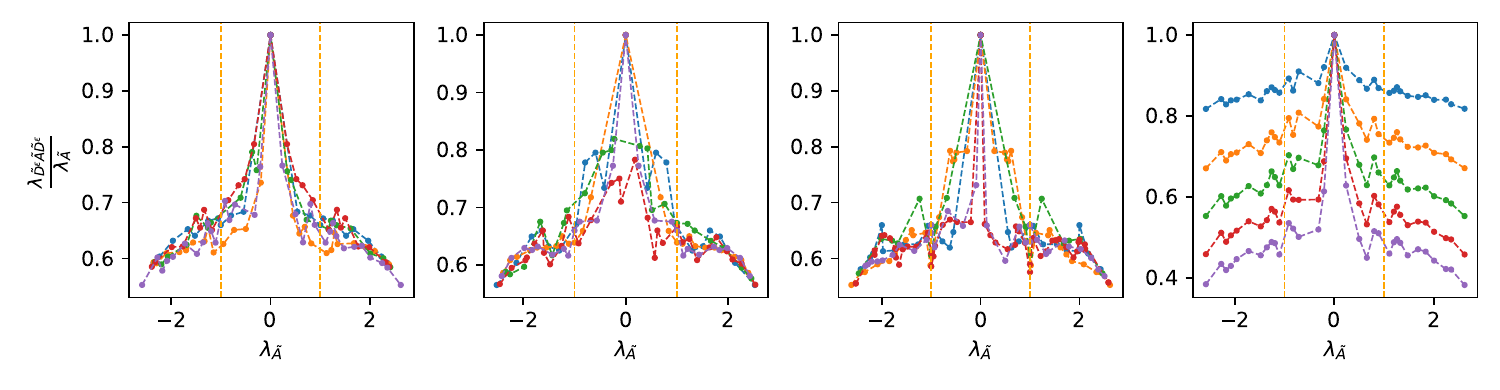}
	\vspace{-20pt}
	\caption{We use the metric $\frac{\lambda_{\tilde D^{\epsilon}\tilde A\tilde D^{\epsilon}}}{\lambda_{\tilde A}}$ to evaluate the shrinking effects of $\tilde D^{\epsilon}\tilde A\tilde D^{\epsilon}$ on the spectrum. We randomly sample 5 graphs in each of three datasets ZINC, MolPCBA and NCI1 respectively. In the first three figures, we use the fixed $\epsilon=-0.3$ on all 5 graphs. In the fourth figure, we use $\epsilon=-0.1, -0.2, -0.3, -0.4, -0.5$ respectively on one graph, which corresponds to the 5 lines from top to bottom.
	More visualization results on other datasets can be found in Appendix~\ref{spectrum_visualizations}.}
	\label{fig:basis_norm}
	\vspace{-10pt}
\end{figure*}
To further investigate the effects of the normalization on the  spectrum, we fix $\eta=0$ and empirically evaluate $\epsilon$ as shown in Fig.~\ref{fig:basis_norm}.
When fixing $\epsilon$, $\tilde D^{\epsilon}\tilde A\tilde D^{\epsilon}$ shrinks the spectrum of $A$ with different degrees on different eigenvalues.
For eigenvalues with small magnitudes (in the middle area of the spectrum), it has a small shrinking effect, while for eigenvalues with large magnitudes, it has a relatively large shrinking effect.
Hence, $\tilde D^{\epsilon}\tilde A\tilde D^{\epsilon}$ can be used as a spectral smoothing method.
Also, different $\epsilon$ results in different shrinking effects, which is consistent with the results in Proposition~\ref{prop:basis_norm}.
Widely-used $\tilde D^{-\frac{1}{2}}\tilde A\tilde D^{-\frac{1}{2}}$ with the spectrum bounded by $[-1, 1]$ may not be a good choice since the diminishing problem.
Intuitively, to utilize more bases, we should narrow the range of the spectrum close to 1 (or -1) to avoid both the diminishing and divergence problems in higher-order bases.
This can vary from different datasets and we should carefully balance $\epsilon$ and $k$.

\section{Related Work}

Many improvements on GNNs can be unified into the spectral smoothing operations, e.g. low-pass filter~\cite{pmlr-v97-wu19e,zhu2021interpreting,klicpera_predict_2019,klicpera2019diffusion,chien2021adaptive,balcilar2021analyzing}, alleviating oversmoothing~\cite{chenWHDL2020gcnii,xu2018representation,liu2020towards,li2018deeper}, graph normalization\cite{cai2020graphnorm}, etc, our analysis on the relations of the correlation issue and the spectrum of underlying graph's matrix provides a unified interpretation on their effectiveness.

ChebyNet~\cite{defferrard2016convolutional}, CayleNet~\cite{8521593}, APPNP~\cite{klicpera_predict_2019}, SSGC~\cite{zhu2020simple}, GPR~\cite{chien2021adaptive}, BernNet~\cite{he2021bernnet}, etc explore various polynomial filters and use the normalized adjacency or Laplacian matrix as basis.
We improve the approximation ability of polynomial filters by altering the spectrum of filter bases.
The resulting bases allow leveraging more bases to approximate more sophisticated filters and are more suitable for learning coefficients from scratch.

We note that the concurrent work \cite{jin2022towards} has also pointed out the overcorrelation issue in the infinite depth case, without further discussion on the reason (e.g. graph's spectrum) behind this phenomenon.
In contrast, we show that correlation is inherently caused by the $unsmooth$ spectrum of the underlying graph filter, and also quantify this effect with spectral smoothness.
It allows to analyze the correlation across all layers instead of only the theoretical infinite depth.

\section{Experiments}

We conduct experiments on TUDatasets~\cite{yanardag2015deep,KKMMN2016}, OGB~\cite{hu2020open} which involve graph classification tasks and ZINC~\cite{dwivedi2020benchmarking} which involves graph regression tasks.
Then, we evaluate the effects of our proposed new graph convolution architecture and two filter bases.

\subsection{Results}

\begin{table}[h]
	\centering
	\caption{Results on TUDatasets. Higher is better.}
	\label{tab:tu_results}
	\vspace{5pt}
	\resizebox{1.0\columnwidth}{!}{%
		\begin{tabular}{ccccc}
			\toprule
			dataset															  & NCI1                            & NCI109                 & ENZYMES                 & PTC\_MR                 	         \\
			\midrule
			GK			 & 62.49$\pm$0.27       & 62.35$\pm$0.3   & 32.70$\pm$1.20      & 55.65$\pm$0.5                  \\
			RW				&-                               &-                         & 24.16$\pm$1.64     & 55.91$\pm$0.3                             			  \\
			PK           & 82.54$\pm$0.5          &-                         &-                           & 59.5$\pm$2.4                             			 \\
			FGSD						 & 79.80                           & 78.84                    &-                            & 62.8                          			 \\
			AWE                  &-                               &-                         & 35.77$\pm$5.93    &-						                  \\
			DGCNN                      & 74.44$\pm$0.47       &-                  	    & 51.0$\pm$7.29       & 58.59$\pm$2.5                  \\
			PSCN                 & 74.44$\pm$0.5         &-                   		 &-                            & 62.29$\pm$5.7                  \\
			\midrule
			DCNN               & 56.61$\pm$1.04        &-                 		   &-                           &-						                   \\
			ECC            & 76.82                          & 75.03                    & 45.67                      &-                             			  \\
			DGK                     & 80.31$\pm$0.46        & 80.32$\pm$0.3   & 53.43$\pm$0.91   & 60.08$\pm$2.6                  \\
			GraphSag                & 76.0$\pm$1.8        &-           & 58.2$\pm$6.0   &-                  \\
			CapsGNN                & 78.35$\pm$1.55        &-           & 54.67$\pm$5.67   &-                  \\
			DiffPool         &76.9$\pm$1.9                               &-                         & 62.53                      &-                              			\\
			GIN                      			    & 82.7$\pm$1.7             &-                    		&-                           & 64.6$\pm$7.0  \\
			$k$-GNN          & 76.2                            &-                         &-                           & 60.9                              			\\
			\midrule
			Spec-GN         &84.79$\pm$1.63 &\textbf{83.62}$\pm$\textbf{0.75} &72.50$\pm$5.79  &\textbf{68.05}$\pm$\textbf{6.41}\\
			Norm-GN          &\textbf{84.87}$\pm$\textbf{1.68} &83.50$\pm$1.27 &\textbf{73.33}$\pm$\textbf{7.96} &67.76$\pm$4.52\\
			\bottomrule
		\end{tabular}
	}
	\vspace{-10pt}
\end{table}

\begin{figure*}[th]
	\centering
	\includegraphics[width=0.95\textwidth]{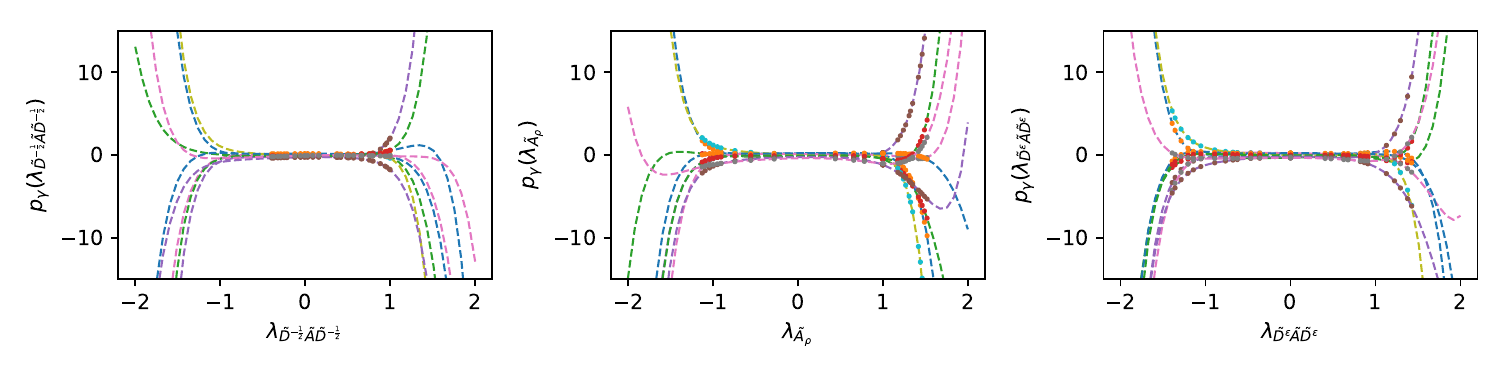}
	\vspace{-10pt}
	\caption{A visualization of the learned filters on ZINC. We tested on three bases with each basis randomly sampling 9 filters. Dots represent the eigenvalues of each basis. More visualization results on other datasets can be found in Appendix~\ref{filter_visualizations}.}
	\label{fig:filter_visual}
	\vspace{-10pt}
\end{figure*}

\textbf{Settings.}
We use the default dataset splits for OGB and ZINC.
For TUDatasets, we follow the standard 10-fold cross-validation protocol and splits from \cite{zhang2018end} and report our results following the protocol described in \cite{xu2018how,ying2018hierarchical}.
Following all baselines on the leaderboard of ZINC, we control the number of parameters around 500K.
The baseline models include: 
GK~\cite{shervashidze2009efficient},
RW~\cite{vishwanathan2010graph},
PK~\cite{neumann2016propagation},
FGSD~\cite{verma2017hunt},
AWE~\cite{pmlr-v80-ivanov18a},
DGCNN~\cite{zhang2018end},
PSCN~\cite{niepert2016learning},
DCNN~\cite{atwood2016diffusion},
ECC~\cite{simonovsky2017dynamic},
DGK~\cite{yanardag2015deep},
CapsGNN~\cite{xinyi2018capsule},
DiffPool~\cite{ying2018hierarchical},
GIN~\cite{xu2018how},
$k$-GNN~\cite{morris2019weisfeiler},
GraphSage~\cite{hamilton2017inductive},
GAT~\cite{velickovic2018graph},
GatedGCN-PE~\cite{bresson2017residual},
MPNN (sum)~\cite{gilmer2017neural},
DeeperG~\cite{li2020deepergcn},
PNA~\cite{corso2020pna},
DGN~\cite{beani2021directional},
GSN~\cite{bouritsas2020improving},
GINE-{\scriptsize VN}~\cite{brossard2020graph},
GINE-{\scriptsize APPNP}~\cite{brossard2020graph},
PHC-GNN~\cite{le2021parameterized},
SAN~\cite{Kreuzer2021rethinking},
Graphormer~\cite{ying2021transformers}.
Spec-GN denotes the proposed graph convolution in Eq.~\ref{equ:channel_wise_graph_conv} with the smoothed filter basis by spectral transformation in Eq.~\ref{equ:basis_optim}.
Norm-GN denotes the proposed graph convolution in Eq.~\ref{equ:channel_wise_graph_conv} with the smoothed filter basis by graph normalization in Eq.~\ref{equ:basis_norm}.

\begin{table}[h]
	\centering
	\caption{Results on ZINC (Lower is better) and MolPCBA (Higher is better).}
	\label{tab:zinc_pcba_results}
	\vspace{5pt}
	\resizebox{1.0\columnwidth}{!}{%
		\begin{tabular}{c|cc}
			\toprule
			method                & ZINC $_{\mathrm{MAE}}$  & MolPCBA $_{\mathrm{AP}}$\\
			\midrule
			GCN &0.367$\pm$0.011 $_{(505\mathrm k)}$ &24.24$\pm$0.34 $_{(2.02\mathrm m)}$ \\
			GIN &0.526$\pm$0.051 $_{(510\mathrm k)}$ &27.03$\pm$0.23 $_{(3.37\mathrm m)}$ \\
			GAT &0.384$\pm$0.007 $_{(531\mathrm k)}$ &- \\
			GraphSage &0.398$\pm$0.002 $_{(505\mathrm k)}$ &- \\
			GatedGCN-PE &0.214$\pm$0.006 $_{(505\mathrm k)}$ &- \\
			MPNN &0.145$\pm$0.007 $_{(481\mathrm k)}$ &- \\
			DeeperG &- &28.42$\pm$0.43 $_{(5.55\mathrm m)}$ \\
			PNA &0.142$\pm$0.010 $_{(387\mathrm k)}$ &28.38$\pm$0.35 $_{(6.55\mathrm m)}$ \\
			DGN &0.168$\pm$0.003 $_{\mathrm{NA}}$ &28.85$\pm$0.30 $_{(6.73\mathrm m)}$ \\
			GSN &0.101$\pm$0.010 $_{(523\mathrm k)}$ &- \\
			GINE-{\scriptsize VN} &- &29.17$\pm$0.15 $_{(6.15\mathrm m)}$ \\
			GINE-{\scriptsize APPNP} &- &\textbf{29.79}$\pm$\textbf{0.30} $_{(6.15\mathrm m)}$ \\
			PHC-GNN &- &29.47$\pm$0.26 $_{(1.69\mathrm m)}$ \\
			SAN &0.139$\pm$0.006 $_{(509\mathrm k)}$ &- \\
			Graphormer &0.122$\pm$0.006 $_{(489\mathrm k)}$ &- \\		
			\midrule
			Spec-GN &\textbf{0.0698}$\pm$\textbf{0.002} $_{(503\mathrm k)}$ &29.65$\pm$0.28 $_{(1.74\mathrm m)}$ \\
			Norm-GN  &0.0709$\pm$0.002 $_{(500\mathrm k)}$ &29.51$\pm$0.33 $_{(1.74\mathrm m)}$ \\
			\bottomrule
		\end{tabular}
	}
	\vspace{-10pt}
\end{table}

\textbf{Results.}
Tab.~\ref{tab:tu_results} and~\ref{tab:zinc_pcba_results} summarize performance of our approaches comparing with baselines on TUDatasets, ZINC and MolPCBA.
For TUDatasets, we report the results of each model in its original paper by default. When the results are not given in the original paper, we report the best testing results given in \cite{zhang2018end,pmlr-v80-ivanov18a,xinyi2018capsule}.
For ZINC and MolPCBA, we report the results of their public leaderboards.
TUDatasets involves small-scale datasets.
NCI1 and NCI109 are around 4K graphs.
ENZYMES and PTC\_MR are under 1K graphs.
General GNNs easily suffer from overfitting on these small-scale data, and therefore we can see that some traditional kernel-based methods even get better performance.
However, Spec-GN and Norm-GN achieve higher classification accuracies by a large margin on these datasets.
The results on TUDatasets show that although Spec-GN and Norm-GN achieve more expressive filters, it does not lead to overfitting on learning graph representations.
Recently, Transformer-based models are quite popular in learning graph representations, and they significantly improve the results on large-scale molecular datasets.
On ZINC, Spec-GN and Norm-GN outperform these Transformer-based models by a large margin.
And on MolPCBA, they are also competitive compared with SOTA results.

\begin{table*}[th]
	\small
	\centering
	\caption{Ablation study results on ZINC with different settings. }
	\label{tab:ablation}
	\vspace{5pt}
	\begin{tabular}{cccccccc}
		\toprule
		\multicolumn{2}{c}{Architecture} & \multicolumn{4}{c}{Basis}  & \multirow{1}{*}{test MAE} & \multirow{1}{*}{valid MAE}\\ \cline{1-2} \cline{3-6}
		shd&idp&$\tilde{D}^{-\frac{1}{2}}\tilde{A}\tilde{D}^{-\frac{1}{2}}$&$\tilde D^{-1}\tilde A$&$\tilde A_{\rho}$&$\tilde{D}^{\epsilon}\tilde{A}\tilde{D}^{\epsilon}$& &\\ \hline
		\yes& &\yes& & & & 0.1415$\pm$0.00748 & 0.1568$\pm$0.00729 \\ \hline
		\yes& & &\yes& & & 0.1439$\pm$0.00900 & 0.1569$\pm$0.00739 \\ \hline
		\yes& & & &\yes& & 0.1061$\pm$0.01018 & 0.1294$\pm$0.01454 \\ \hline
		\yes& & & & &\yes& 0.1133$\pm$0.01711 & 0.1316$\pm$0.02057 \\ \hline
		&\yes&\yes& & & & 0.0944$\pm$0.00379 & 0.1100$\pm$0.00787 \\ \hline
		&\yes& &\yes& & & 0.0982$\pm$0.00417 & 0.1172$\pm$0.00666 \\ \hline
		&\yes& & &\yes& & \textbf{0.0698}$\pm$\textbf{0.00200} & \textbf{0.0884}$\pm$\textbf{0.00319} \\ \hline
		&\yes& & & &\yes & 0.0709$\pm$0.00176 & 0.0929$\pm$0.00445 \\
		\bottomrule
	\end{tabular}
	\vspace{-10pt}
\end{table*}

\subsection{Ablation Studies}
\label{sec:ablation_studies}

We perform ablation studies on the proposed architecture and the filter bases $\tilde A_{\rho}$ (by setting $S=\tilde A$ in Eq.~\ref{equ:basis_optim}) and $\tilde{D}^{\epsilon}\tilde{A}\tilde{D}^{\epsilon}$ on ZINC.
We use ``idp'' and ``shd'' to respectively represent the correlation-free architecture (also known as independent filter architecture) in Eq.~\ref{equ:channel_wise_graph_conv} and the general shared filter architecture in Eq.~\ref{equ:gnn_generalization}.
Both architectures learn the filter coefficients from scratch.

\textbf{Correlation-free architecture and different filter bases.}
In Fig.~\ref{fig:filter_visual}, we visualize the learned filters in the correlation-free on three bases, i.e. $\tilde{D}^{-\frac{1}{2}}\tilde{A}\tilde{D}^{-\frac{1}{2}}$, $\tilde A_{\rho}$ and $\tilde{D}^{\epsilon}\tilde{A}\tilde{D}^{\epsilon}$.
The visualizations show that each channel indeed learns a different filter on all three bases.
$\tilde{D}^{-\frac{1}{2}}\tilde{A}\tilde{D}^{-\frac{1}{2}}$ has the bounded spectrum $[-1, 1]$ that is slightly close to $1$ due to the involvement of self-loop.
The filters learn a similar response on all range which corresponds to different frequencies in frequency domain.
$\tilde A_{\rho}$ and $\tilde{D}^{\epsilon}\tilde{A}\tilde{D}^{\epsilon}$ have the spectrum close to $1$ or $-1$ while the filters learn diverse responses on these areas, which corresponds to more complex patterns on different frequencies.
Tab.~\ref{tab:ablation} shows that the correlation-free always outperforms the shared filter by a large margin on all tested bases.
Both $\tilde{D}^{-\frac{1}{2}}\tilde{A}\tilde{D}^{-\frac{1}{2}}$ and $\tilde D^{-1}\tilde A$ have the bounded spectrum $[-1, 1]$ and they have similar performance.
$\tilde A_{\rho}$ and $\tilde{D}^{\epsilon}\tilde{A}\tilde{D}^{\epsilon}$ narrow the range of the spectrum close to $1$ or $-1$ through completely different strategies, but they have similar performance that is much better than $\tilde{D}^{-\frac{1}{2}}\tilde{A}\tilde{D}^{-\frac{1}{2}}$ and $\tilde D^{-1}\tilde A$.
This validates our analysis on the filter basis.
Meanwhile, $\tilde A_{\rho}$ achieves more accurate control on the spectrum, and correspondingly, it slightly outperforms $\tilde{D}^{\epsilon}\tilde{A}\tilde{D}^{\epsilon}$.

\begin{figure*}[h]
	\centering
	\includegraphics[width=0.85\textwidth]{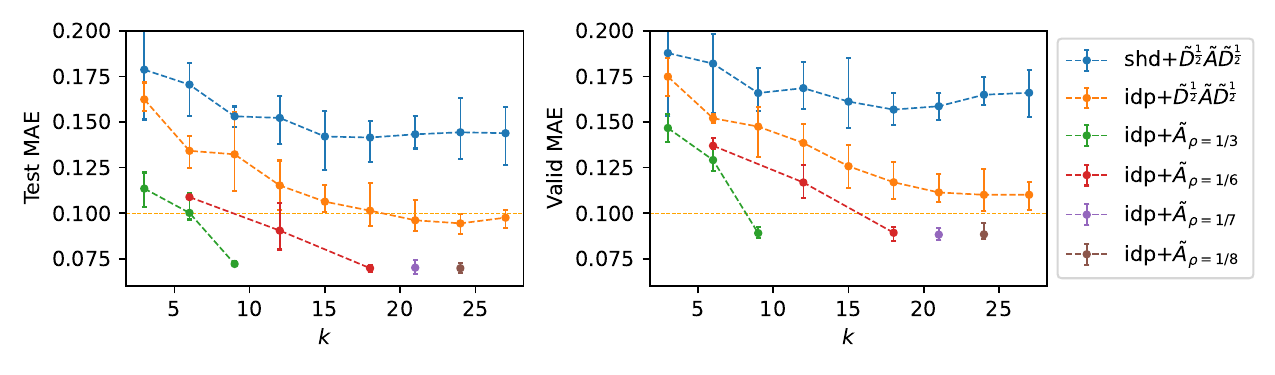}
	\vspace{-10pt}
	\caption{Ablation study results on ZINC with different number of bases $k$.}
	\label{fig:ablation_plot}
	\vspace{-10pt}
\end{figure*}

\textbf{Do more bases gain improvements?}
In Fig.~\ref{fig:ablation_plot}, we systematically evaluate the effects of the number of bases on learning graph representations, including $\tilde{D}^{-\frac{1}{2}}\tilde{A}\tilde{D}^{-\frac{1}{2}}$ and our $\tilde A_{\rho}$ with $\rho=1/3, 1/6, 1/7, 1/8$.
The shared filter case, i.e. shd$+\tilde{D}^{-\frac{1}{2}}\tilde{A}\tilde{D}^{-\frac{1}{2}}$ cannot well leverage more bases (a larger $k$) as the MAE stops decreasing at $0.150$ which is also reported by several baselines in Tab.~\ref{tab:zinc_pcba_results}.
In contrast, both correlation-free cases idp$+\tilde{D}^{-\frac{1}{2}}\tilde{A}\tilde{D}^{-\frac{1}{2}}$ and idp$+\tilde A_{\rho}$ outperform the shared filter case by a large margin and they continuously gain improvements when increasing $k$.
The MAE of idp$+\tilde{D}^{-\frac{1}{2}}\tilde{A}\tilde{D}^{-\frac{1}{2}}$ stops decreasing at the test MAE close to 0.09 and the valid MAE close to 0.11.
By replacing $\tilde{D}^{-\frac{1}{2}}\tilde{A}\tilde{D}^{-\frac{1}{2}}$ with $\tilde A_{\rho}$, the best test MAE is below 0.07, and the best valid MAE is close to 0.088.
The bases in $\tilde A_{\rho}$ are controlled by both $\rho$ and $k$.
We use the tuple $(\rho, k)$ to denote a combination of $\rho$ and $k$.
By fixing $\rho$, the curves corresponding to $\rho=1/3$ and $\rho=1/6$ show that increasing $k$ gains improvements.
By fixing the upper bound of $\rho\times k$ to be 1, $(1/6, 6)$ involves 3 more bases than $(1/3, 3)$ and outperforms $(1/3, 3)$.
The same results are also reflected in the comparison of $(1/6, 12)$ and $(1/3, 6)$.
For the comparison of $(1/6, 18)$ and $(1/3, 9)$, both settings achieve the lowest MAE and the difference is less obvious.

\begin{figure*}[h]
	\centering
	\includegraphics[width=0.8\textwidth]{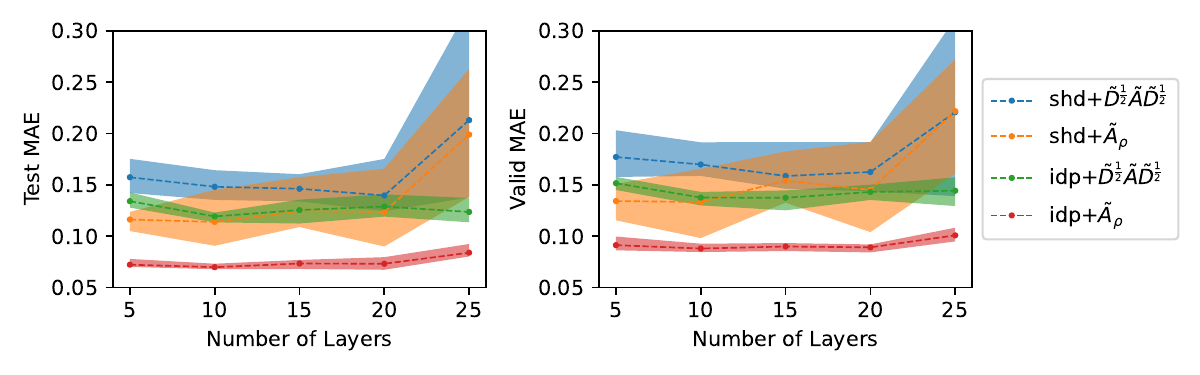}
	\vspace{-10pt}
	\caption{Ablation study results on ZINC with different number of layers.}
	\label{fig:zinc_depth}
	\vspace{-10pt}
\end{figure*}

\textbf{The effects of model depth.}
Fig.\ref{fig:zinc_depth} shows the performance comparisons between correlation-free and shared filter as depth increases.
Each architecture is tested with the default basis $\tilde D^{\frac{1}{2}}\tilde A\tilde D^{\frac{1}{2}}$ and our proposed $\tilde A_{\rho}$.
We set the same number of bases in all resulting models, and each model is tested with the number of layers (depth) equal to $\{5, 10, 15, 20, 25\}$.
The results show that the correlation-free can preserve the performance as depth increases.
The shared filter cases perform quite unstable and drop dramatically when depth $>20$. Also, across all depths, the correlation-free almost always outperforms the shared filter and has low variance among different runs.
In Appendix~\ref{ablate_deep}, we also test cosine similarities of different layers in a deep model.

\textbf{Stability.}
We also found that the correlation-free is more stable in different runs than the shared filter case as reflected in the standard deviation in Tab.~\ref{tab:ablation}.
This is probably because different channels may pose different patterns, which causes interference among each other in the shared filter case. While the correlation-free well avoids this problem.
Also, the results of $\tilde A_{\rho}$ and $\tilde{D}^{\epsilon}\tilde{A}\tilde{D}^{\epsilon}$ are more stable than $\tilde{D}^{-\frac{1}{2}}\tilde{A}\tilde{D}^{-\frac{1}{2}}$ and $\tilde D^{-1}\tilde A$ in different runs.
For $\tilde{D}^{-\frac{1}{2}}\tilde{A}\tilde{D}^{-\frac{1}{2}}$ and $\tilde D^{-1}\tilde A$, the difference between the best and the worst runs can be more than 0.02.
While for $\tilde A_{\rho}$ and $\tilde{D}^{\epsilon}\tilde{A}\tilde{D}^{\epsilon}$, this difference is less than 0.01.
More results are given in Appendix~\ref{more_results}.
The instability of $\tilde{D}^{-\frac{1}{2}}\tilde{A}\tilde{D}^{-\frac{1}{2}}$ and $\tilde D^{-1}\tilde A$ is probably because learning filter coefficients from scratch without any constraints is difficult to maintain spectrum properties and therefore easily falls into an ill-posed filter~\cite{he2021bernnet}.
In contrast, $\tilde A_{\rho}$ and $\tilde{D}^{\epsilon}\tilde{A}\tilde{D}^{\epsilon}$ inherently with smoother spectrum alleviate this problem and make them more appropriate in the scenario of learning coefficients from scratch.

\section{Conclusion}

We study the effects of spectrum in GNNs. It shows that in existing architectures, the unsmooth spectrum results in the correlation issue, which acts as the obstacle to developing deep models as well as applying more powerful graph filters. Based on this observation, we propose the correlation-free architecture which decouples the correlation issue from filter design. Then, we show that the spectral characteristics also hinder the approximation abilities of polynomial filters and address it by altering the graph's spectrum.
Our extensive experiments show the significant performance gain of correlation-free architecture with powerful filters.

\section*{Acknowledgments}

This work is supported in part by the National Key Research and Development Program of China (no. 2021ZD0112400), and also in part by the National Natural Science Foundation of China under grants U1811463 and 62072069.

\nocite{langley00}

\bibliography{reference}
\bibliographystyle{icml2022}


\newpage
\appendix
\onecolumn

\section{Derivations of Eq.~\ref{equ:cos_signal} and Eq.~\ref{equ:cos_signal_eigenvalue}}
\label{deriv:equ:cos_signal_eigenvalue}

Since $\mathcal S\in\mathbb R^{n\times n}$ is a symmetric matrix, assume the eigendecomposition $\mathcal S=P\Lambda P^{\top}$ with $P=\bigl(\bm p_1, \bm p_2, \dots, \bm p_n\bigr)$ and $\|\bm p_i\|=1, i\in [n]$.

\begin{equation}
	\nonumber
	\begin{aligned}
		\cos\bigl(\langle \bm h, \bm p_i\rangle\bigr)
		&=\frac{\bm h^{\top} \bm p_i}{\|\bm h\|\|\bm p_i\|}\\
		&=\frac{\bm h^{\top} \bm p_i}{\|\bm h\|}\\
		&=\frac{\bm h^{\top} \bm p_i}{\sqrt{\bm h^{\top} \bm h}}\\
		&=\frac{\bm h^{\top} \bm p_i}{\sqrt{(P^{\top}\bm h)^{\top} P^{\top}\bm h}}\\
		&=\frac{\bm h^{\top}\bm p_i}
		{\sqrt{\sum^n_{j=1}\bigl(\bm p_j^{\top}\bm h\bigr)^2}}\\
		&=\frac{\bm h^{\top}\bm p_i}
		{\sqrt{\sum^n_{j=1}\bigl(\bm h^{\top}\bm p_j\bigr)^2}}\\
		&=\frac{\alpha_i}{\sqrt{\sum^n_{j=1}\alpha_j^2}}.
	\end{aligned}
\end{equation}

\begin{equation}
	\nonumber
	\begin{aligned}
		\cos\bigl(\langle \mathcal S\bm h, \bm p_i\rangle\bigr)&=\frac{\bigl(\mathcal S\bm h\bigr)^{\top} \bm p_i}{\| \mathcal S\bm h\|\|\bm p_i\|}\\
		&=\frac{\bigl(\mathcal S\bm h\bigr)^{\top} \bm p_i}{\sqrt{\bigl(\mathcal S\bm h\bigr)^{\top} \mathcal S\bm h}}\\
		&=\frac{\bigl(P\Lambda\bigl(P^{\top}\bm h\bigr)\bigr)^{\top} \bm p_i}
		{\sqrt{\bigl(P\Lambda\bigl(P^{\top}\bm h\bigr)\bigr)^{\top}\bigl(P\Lambda\bigl(P^{\top}\bm h\bigr)\bigr)}}\\
		&=\frac{\bigl(P^{\top}\bm h\bigr)^{\top}\Lambda P^{\top} \bm p_i}
		{\sqrt{\bigl(P^{\top}\bm h\bigr)^{\top}\Lambda^2\bigl(P^{\top}\bm h\bigr)}}\\
		&=\frac{
			\begin{pmatrix}
				\bm p_1^{\top}\bm h,\dots,\bm p_i^{\top}\bm h,\dots,\bm p_n^{\top}\bm h
			\end{pmatrix}
			\begin{pmatrix}
				\lambda_1 \\
				& \ddots \\
				& & \lambda_i \\
				& & & \ddots \\
				& & & & \lambda_n \\
			\end{pmatrix}
			\begin{pmatrix}
				\bm p_1^{\top} \\
				\vdots \\
				\bm p_i^{\top} \\
				\vdots \\
				\bm p_n^{\top} \\
			\end{pmatrix}
			\begin{pmatrix}
				\bm p_i
			\end{pmatrix}
		}{
			\sqrt{\bigl(P^{\top}\bm h\bigr)^{\top}\Lambda^2\bigl(P^{\top}\bm h\bigr)}} \\
		&=\frac{\bm p_i^{\top}\bm h \lambda_i}
		{\sqrt{\sum^n_{j=1}\bigl(\bm p_j^{\top}\bm h\bigr)^2\lambda_j^2}} \\
		&=\frac{\bm h^{\top}\bm p_i \lambda_i}
		{\sqrt{\sum^n_{j=1}\bigl(\bm h^{\top}\bm p_j\bigr)^2\lambda_j^2}}\\
		&=\frac{\alpha_i\lambda_i}{\sqrt{\sum^n_{j=1}\alpha_j^2\lambda_j^2}}
	\end{aligned}
\end{equation}

\section{Proof of Proposition~\ref{prop:cos_convergence}}
\label{proof:prop:cos_convergence}

\begin{proof}
	(i)
	As $ \mathcal S^k= P\Lambda^k P^{\top}$ and Eq.~\ref{equ:cos_signal_eigenvalue}, for $k=0,1,2,\dots,+\infty$, we have
	\begin{equation}
		\nonumber
		\begin{aligned}
			|\cos(\langle \mathcal S^k\bm h, \bm p_1\rangle)|&=\frac{|\alpha_1\lambda_1^k|}{\sqrt{\sum^n_{i=1}\alpha_i^2\lambda_i^{2k}}}\\
			&=\frac{|\lambda_1|}{|\lambda_1|} \frac{|\alpha_1\lambda_1^k|}{\sqrt{\sum^n_{i=1}\alpha_i^2\lambda_i^{2k}}}\\
			&=\frac{|\alpha_1\lambda_1^{k+1}|}{\sqrt{\lambda_1^2\sum^n_{i=1}\alpha_i^2\lambda_i^{2k}}}\\
			&\leq\frac{|\alpha_1\lambda_1^{k+1}|}{\sqrt{\sum^n_{i=1}\alpha_i^2\lambda_i^{2(k+1)}}}\\
			&=|\cos(\langle \mathcal S^{k+1}\bm h, \bm p_1\rangle)|.
		\end{aligned}
	\end{equation}
	Similarly, we can prove that $|\cos(\langle \mathcal S^k\bm h, \bm p_n\rangle)|\geq |\cos(\langle \mathcal S^{k+1}\bm h, \bm p_n\rangle)|$.
	
	(ii)
	Since $|\cos(\langle \mathcal S^k\bm h, \bm p_n\rangle)|$ monotonously increases with respect to $k$ and has the upper bound 1, $|\cos(\langle \mathcal S^k\bm h, \bm p_n\rangle)|$ must be convergent.
	\begin{equation}
		\nonumber
		\begin{aligned}
			\lim\limits_{k\rightarrow\infty}|\cos(\langle \mathcal S^k\bm h, \bm p_1\rangle)|&=\lim\limits_{k\rightarrow\infty}\frac{|\alpha_1\lambda_1^k|}{\sqrt{\sum^n_{i=1}\alpha_i^2\lambda_i^{2k}}}\\
			&=\lim\limits_{k\rightarrow\infty}\frac{|\alpha_1|}{\sqrt{\alpha_1^2+\sum^n_{i=2}\alpha_i^2\bigl(\frac{\lambda_i}{\lambda_1}\bigr)^{2k}}}\\
			&=\frac{|\alpha_1|}{\sqrt{\alpha_1^2+\lim\limits_{k\rightarrow\infty}\sum^n_{i=2}\alpha_i^2\bigl(\frac{\lambda_i}{\lambda_1}\bigr)^{2k}}}
		\end{aligned}
	\end{equation}
	As $|\lambda_1|>|\lambda_2|\geq,\dots,\geq|\lambda_n|$, we have $\lim\limits_{k\rightarrow\infty}\sum^n_{i=2}\alpha_i^2\bigl(\frac{\lambda_i}{\lambda_1}\bigr)^{2k}=0$ and the convergence speed is decided by $|\frac{\lambda_2}{\lambda_1}|$.
	Therefore $\lim\limits_{k\rightarrow\infty}|\cos(\langle \mathcal S^k\bm h, \bm p_1\rangle)|=1$.
	
	\begin{equation}
		\nonumber
		\begin{aligned}
			\cos\bigl(\langle \mathcal S\bm h, \mathcal S\bm h^{\prime}\rangle\bigr)&=\frac{\bigl(\mathcal S\bm h\bigr)^{\top} \mathcal S\bm h^{\prime}}{\| \mathcal S\bm h\|\| \mathcal S\bm h^{\prime}\|}\\
			&=\frac{\bigl(\mathcal S\bm h\bigr)^{\top} \mathcal S\bm h^{\prime}}
			{\sqrt{\bigl(\mathcal S\bm h\bigr)^{\top} \mathcal S\bm h}\sqrt{\bigl(\mathcal S\bm h^{\prime}\bigr)^{\top} \mathcal S\bm h^{\prime}}}\\
			&=\frac{\bigl(P\Lambda\bigl(P^{\top}\bm h\bigr)\bigr)^{\top} P\Lambda\bigl(P^{\top}\bm h^{\prime}\bigr)}
			{\sqrt{\bigl(P\Lambda\bigl(P^{\top}\bm h\bigr)\bigr)^{\top}\bigl(P\Lambda\bigl(P^{\top}\bm h\bigr)\bigr)}\sqrt{\bigl(P\Lambda\bigl(P^{\top}\bm h^{\prime}\bigr)\bigr)^{\top}\bigl(P\Lambda\bigl(P^{\top}\bm h^{\prime}\bigr)\bigr)}}\\
			&=\frac{\bigl(P^{\top}\bm h\bigr)^{\top}\Lambda^2 P^{\top}\bm h^{\prime}}
			{\sqrt{\bigl(P^{\top}\bm h\bigr)^{\top}\Lambda^2\bigl(P^{\top}\bm h\bigr)}\sqrt{\bigl(P^{\top}\bm h^{\prime}\bigr)^{\top}\Lambda^2\bigl(P^{\top}\bm h^{\prime}\bigr)}}\\
			&=\frac{\boldsymbol\alpha^{\top}\Lambda^2\boldsymbol\beta}
			{\sqrt{\boldsymbol\alpha^{\top}\Lambda^2\boldsymbol\alpha}\sqrt{\boldsymbol\beta^{\top}\Lambda^2\boldsymbol\beta}}\\
			&=\frac{\sum^n_{i=1}\alpha_i\beta_i\lambda_i^2}{\sqrt{\sum^n_{i=1}\alpha_i^2\lambda_i^2}\sqrt{\sum^n_{i=1}\beta_i^2\lambda_i^2}}
		\end{aligned}
	\end{equation}
	Then, 
	\begin{equation}
		\nonumber
		\begin{aligned}
			\lim\limits_{k\rightarrow\infty}|\cos(\langle \mathcal S^k\bm h, \mathcal S^k\bm h^{\prime}\rangle)|&=\lim\limits_{k\rightarrow\infty}\frac{\bigl|\sum^n_{i=1}\alpha_i\beta_i\lambda_i^{2k}\bigr|}{\sqrt{\sum^n_{i=1}\alpha_i^2\lambda_i^{2k}}\sqrt{\sum^n_{i=1}\beta_i^2\lambda_i^{2k}}}\\
			&=\lim\limits_{k\rightarrow\infty}\frac{\bigl|\sum^n_{i=1}\alpha_i\beta_i\frac{\lambda_i}{\lambda_1}^{2k}\bigr|}{\sqrt{\sum^n_{i=1}\alpha_i^2\frac{\lambda_i}{\lambda_1}^{2k}}\sqrt{\sum^n_{i=1}\beta_i^2\frac{\lambda_i}{\lambda_1}^{2k}}}\\			&=\lim\limits_{k\rightarrow\infty}\frac{\bigl|\alpha_1\beta_1+\sum^n_{i=2}\alpha_i\beta_i\frac{\lambda_i}{\lambda_1}^{2k}\bigr|}{\sqrt{\alpha_1^2+\sum^n_{i=2}\alpha_i^2\frac{\lambda_i}{\lambda_1}^{2k}}\sqrt{\beta_1^2+\sum^n_{i=2}\beta_i^2\frac{\lambda_i}{\lambda_1}^{2k}}}\\
			&=\frac{\bigl|\alpha_1\beta_1\bigr|}{\sqrt{\alpha_1^2}\sqrt{\beta_1^2}}\\
			&=1
		\end{aligned}
	\end{equation}
\end{proof}

\section{More Discussions of Spectral Optimization on Filter Basis}
\label{more_filter_basis}

We use $\textrm E_{(S, \lambda)}$ to denote the eigenspace of $S$ associated with $\lambda$ such that $\textrm E_{(S, \lambda)}=\{\bm v: (S-\lambda I)\bm v=\bm 0\}$.

\begin{proposition}
	\label{prop:eig_mapping}
	Given a symmetric matrix $S\in\mathbb R^{n\times n}$ with $S=P\Lambda P^{\top}$ where $\Lambda=\textrm{diag}(\lambda_1, \lambda_2, \dots, \lambda_n)$, and $P$ can be any eigenbasis of $S$,
	let $S_{\phi}=P\phi(\Lambda)P^{\top}$, where $\phi(\cdot)$ is an entry-wise function applied on $\Lambda$. Then we have \\
	(i) $\textrm E_{(S, \lambda_i)}\subseteq\textrm E_{(S_{\phi}, \phi(\lambda_i))}, i\in[n]$; \\
	(ii) Meanwhile, if $\phi(\cdot)$ is injective, $\textrm E_{(S, \lambda_i)}=\textrm E_{(S_{\phi}, \phi(\lambda_i))}$ and $\mathcal F_{\phi}(S)=P\phi(\Lambda)P^{\top}$ is injective.
\end{proposition}

\begin{proof}
	
	Let $P=(\bm p_1, \bm p_2, \dots, \bm p_n)$.
	$S=P\Lambda P^{\top}$ is equivalent to $S\bm p_i=\lambda_i\bm p_i, i\in[n]$.
	For any $i\in[n]$, the geometric multiplicity of any $\lambda_i$ is equal to its algebraic multiplicity, and $\textrm E_{(S, \lambda_i)}=\textrm{Span}(\{\bm p_k|\lambda_k=\lambda_i, k\in[n]\})$.
	$S_{\phi}=P\phi(\Lambda)P^{\top}$ and $S_{\phi}\bm p_i=\phi(\lambda_i)\bm p_i, i\in[n]$.
	Similarly, for any $i\in[n]$, $\textrm E_{(S_{\phi}, \phi(\lambda_i))}=\textrm{Span}(\{\bm p_k|\phi(\lambda_k)=\phi(\lambda_i), k\in[n]\})$.
	Note that $\{\bm p_k|\lambda_k=\lambda_i, k\in[n]\}\subseteq\{\bm p_k|\phi(\lambda_k)=\phi(\lambda_i), k\in[n]\}$ for any $i\in[n]$.
	Hence $\textrm{Span}(\{\bm p_k|\lambda_k=\lambda_i, k\in[n]\})\subseteq\textrm{Span}(\{\bm p_k|\phi(\lambda_k)=\phi(\lambda_i), k\in[n]\}$.
	As a result, $\textrm E_{(S, \lambda_i)}\subseteq\textrm E_{(S_{\phi}, \phi(\lambda_i))}$ for any $i\in[n])$.
	
	If $\phi(\cdot)$ is injective, $\{\bm p_k|\lambda_k=\lambda_i, k\in[n]\}=\{\bm p_k|\phi(\lambda_k)=\phi(\lambda_i), k\in[n]\}$ for any $i\in[n]$.
	Thus $\textrm E_{(S, \lambda_i)}=\textrm E_{(S_{\phi}, \phi(\lambda_i))}$.
	
	We use $\sigma(S)$ to denote the generalisation of the set of all eigenvalues of $S$ (Slso known as the spectrum of $S$).
	Let $S=P\Lambda_1 P^{\top}$ and $B=Q\Lambda_2 Q^{\top}$.
	Suppose $S\neq B$,
	to prove $S_{\phi}=\mathcal F_{\phi}(S)\neq B_{\phi}=\mathcal F_{\phi}(B)$, we discuss two cases respectively.
	
	\textbf{Case 1:} $\sigma(S)\neq\sigma(B)$
	
	Then $\sigma(S_{\phi})\neq\sigma(B_{\phi})$.
	The characteristic polynomials of $S_{\phi}$ and $B_{\phi}$ are different.
	Therefore, $S_{\phi}\neq B_{\phi}$.
	
	\textbf{Case 2:} $\sigma(S)=\sigma(B)$
	
	Then $\Lambda_1=\Lambda_2=\Lambda$.
	We prove the equivalent proposition "$S_{\phi}=B_{\phi}\Rightarrow S=B$".
	If $S_{\phi}=B_{\phi}$, $P\phi(\Lambda)P^{\top}=Q\phi(\Lambda)Q^{\top}$.
	For any $\lambda_i$ with geometric multiplicity $k$, we can find the corresponding eigenvectors $\bm p_1, \bm p_2, \dots, \bm p_k$ according to $P\phi(\Lambda)P^{\top}$.
	Similarly, we can find the corresponding eigenvectors $\bm q_1, \bm q_2, \dots, \bm q_k$ according to $Q\phi(\Lambda)Q^{\top}$.
	Note that the eigen-decomposition is unique in terms of eigenspaces.
	Thus, $\textrm E_{(S_{\phi}, \phi(\lambda_i))}=\textrm{Span}(\bm p_1, \bm p_2, \dots, \bm p_k)=\textrm{Span}(\bm q_1, \bm q_2, \dots, \bm q_k)=\textrm E_{(B_{\phi}, \phi(\lambda_i))}$.
	Therefore, for any $\lambda_i$, $\textrm E_{(S, \lambda_i)}=\textrm E_{(B, \lambda_i)}$ (As given in Proposition~\ref{prop:eig_mapping}).
	Correspondingly, $S=P\Lambda P^{\top}=Q\Lambda Q^{\top}=B$.
	
\end{proof}

Proposition~\ref{prop:eig_mapping} shows that the eigenspace of $S_{\phi}$ involves the eigenspace of $S$.
Therefore, $S_{\phi}$ is invariant to the choice of eigenbasis, i.e., $S_{\phi}=P\phi(\Lambda)P^{\top}=P^{\prime}\phi(\Lambda)P^{\prime \top}$ for any eigenbases $P$ and $P^{\prime}$ of $S$.
Hence, $S_{\phi}$ is unique to $S$ for a given $\phi(\cdot)$.
Consistently, we denote the mapping $\mathcal F_{\phi}(S)=\mathcal F_{\phi}(P\Lambda P^{\top})=P\phi(\Lambda)P^{\top}$.

When $\mathcal F_{\phi}$ is injective, $\mathcal F_{\phi}(S)$ and $S$ share the same algebraic multiplicity.
Otherwise, $\mathcal F_{\phi}(S)$ has a larger algebraic multiplicity on the corresponding eigenvalues, which may weaken the approximation ability based on the understanding of Vandermonde matrix.
Also, the injectivity of $\mathcal F_{\phi}$ serves as a guarantee that the transformation is reversible with no information loss.

$\mathcal F_{\phi}(\cdot)$ is also equivariant to graph isomorphism.
For any two graphs $G_1$ and $G_2$ with matrix representations $S_1$ and $S_2$ (e.g., adjacency matrix, Laplacian matrix, etc.), $G_1$ and $G_2$ are isomorphic if and only if there exists a permutation matrix $M$ such that $MS_1 M^{\top}=S_2$.
We denote $I(S)=MSM^{\top}$.
Then
\begin{claim}
	$\mathcal F_{\phi}(\cdot)$ is equivariant to graph isomorphism, i.e. $\mathcal F_{\phi}(I(S))=I(\mathcal F_{\phi}(S))$.
\end{claim}
\begin{proof}
	\begin{equation}
		\nonumber
		\begin{aligned}
			\mathcal F_{\phi}(I(S))&=\mathcal F_{\phi}(MSM^{\top})\\
			&=\mathcal F_{\phi}(M(P\Lambda P^{\top})M^{\top})\\
			&=\mathcal F_{\phi}((MP)\Lambda(MP)^{\top})\\
			&=(MP)\phi(\Lambda)(MP)^{\top}\\
			&=M(P\phi(\Lambda)P^{\top})M^{\top}\\
			&=I(\mathcal F_{\phi}(S))
		\end{aligned}
	\end{equation}
\end{proof}
Hence, for a specific GNN model $f_{\textrm{GNN}}$, $f_{\textrm{GNN}}(\mathcal F_{\phi}(I(S)))=f_{\textrm{GNN}}(I(\mathcal F_{\phi}(S)))=f_{\textrm{GNN}}(\mathcal F_{\phi}(S))$.
The learned representation is invariant to graph isomorphism (also known as permutation invariance~\cite{NIPS2017_f22e4747,murphy2018janossy}) when introducing $\mathcal F_{\phi}(\cdot)$.

\section{Proof of Proposition~\ref{prop:basis_norm}}
\label{proof:prop:basis_norm}

\begin{proof}
	Let $\mathring A=(D+\eta I)^{\epsilon}(A+\eta I)(D+\eta I)^{\epsilon}$.
	According to Courant-Fischer theorem,
	\begin{equation}
		\nonumber
		\begin{aligned}
			\mu_i=\min_{\mathrm{dim}(S)=i} \max_{x\in S}\frac{\bm x^{\top}\mathring A\bm x}{\bm x^{\top}\bm x}.
		\end{aligned}
	\end{equation}
	Let $\bm y=(D+\eta I)^{\epsilon}\bm x$. As the change of variables $\bm y=(D+\eta I)^{\epsilon}\bm x$ is non-singular, this is equivalent to
	\begin{equation}
		\nonumber
		\begin{aligned}
			\mu_i=\min_{\mathrm{dim}(T)=i} \max_{\bm y\in T}\frac{\bm y^{\top}(A+\eta I)\bm y}{\bm y^{\top}(D+\eta I)^{-2\epsilon}\bm y}.
		\end{aligned}
	\end{equation}
	Therefore,
	\begin{equation}
		\nonumber
		\begin{aligned}
			\mu_i&=\min_{\mathrm{dim}(T)=i} \max_{\bm y\in T}\frac{\bm y^{\top}(A+\eta I)\bm y}{\bm y^{\top}(D+\eta I)^{-2\epsilon}\bm y}\\
			&\geq\min_{\mathrm{dim}(T)=i} \max_{\bm y\in T}\frac{\bm y^{\top}(A+\eta I)\bm y}{(d_{\mathrm{max}}+\eta)^{-2\epsilon}\bm y^{\top}\bm y}\\
			&=(d_{\mathrm{max}}+\eta)^{2\epsilon}\bigl(\min_{\mathrm{dim}(T)=i} \max_{\bm y\in T}\frac{\bm y^{\top}A\bm y}{\bm y^{\top}\bm y}+\eta\bigr)\\
			&=(\lambda_i+\eta)(d_{\mathrm{max}}+\eta)^{2\epsilon}.
		\end{aligned}
	\end{equation}
	Similarly, we can prove $\mu_i\leq(\lambda_i+\eta)(d_{\mathrm{min}}+\eta)^{2\epsilon}$.
\end{proof}

\FloatBarrier
\section{Visualizations of the Effects of the Normalization $\tilde D^{\epsilon}\tilde A\tilde D^{\epsilon}$ on the Spectrum}
\label{spectrum_visualizations}

\begin{figure}[h]
	\centering
	\includegraphics[width=0.95\textwidth]{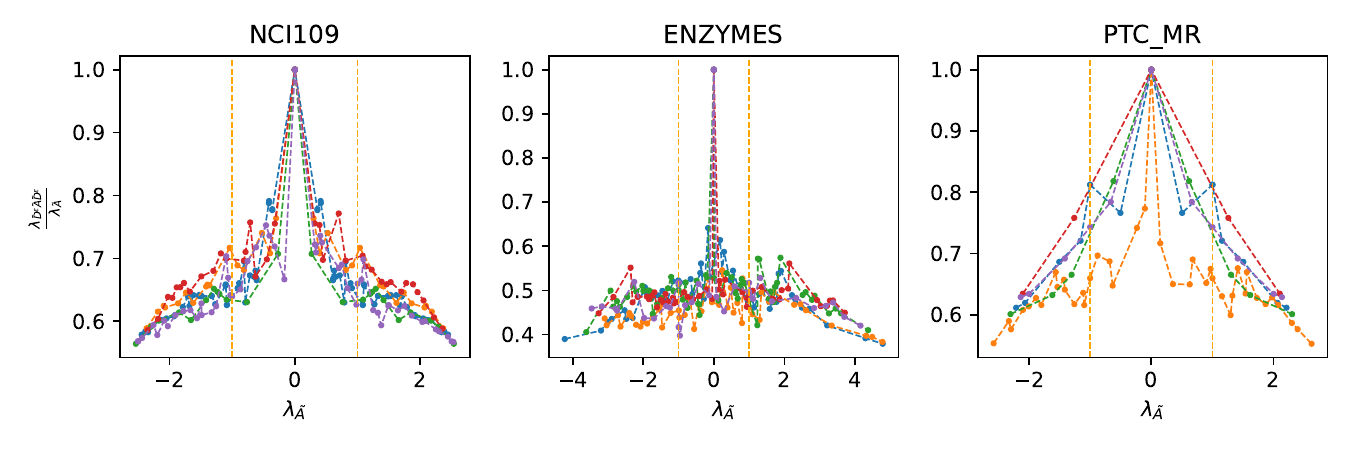}
	\vspace{-10pt}
	\caption{We randomly sample 5 graphs in each of three datasets NCI109, ENZYMES and PTC\_MR respectively. And we use the fixed $\epsilon=-0.3$ to see the effects of the normalization on all graphs.}
	\label{fig:spectrum_visualizations}
\end{figure}

\FloatBarrier
\section{Visualizations of the Learned Filters}
\label{filter_visualizations}

\begin{figure}[h]
	\centering
	\includegraphics[width=0.95\textwidth]{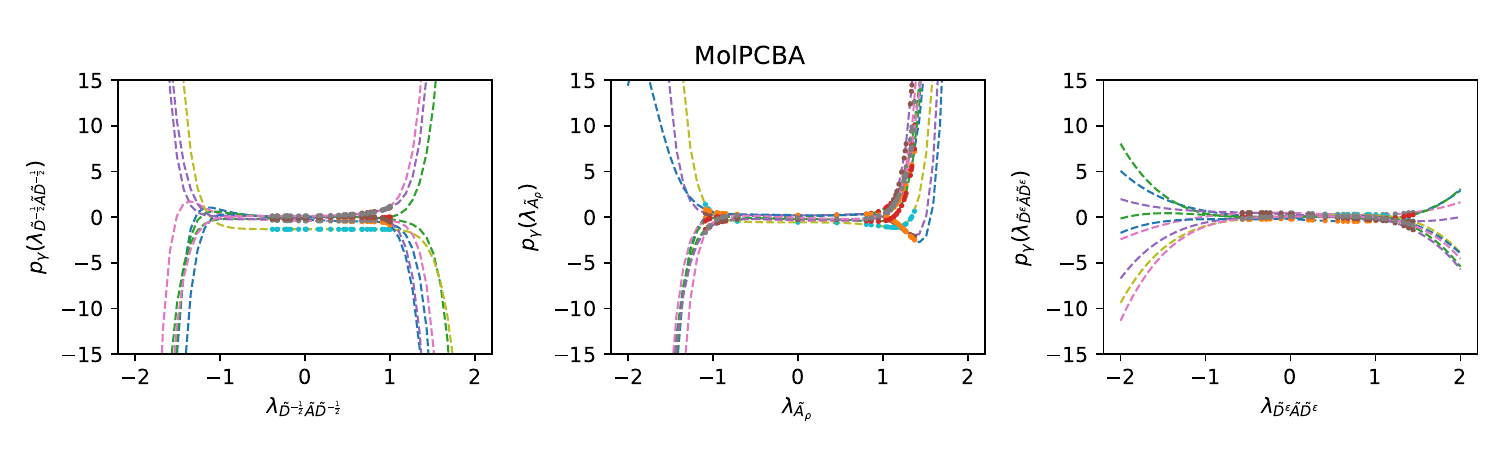}
	\includegraphics[width=0.95\textwidth]{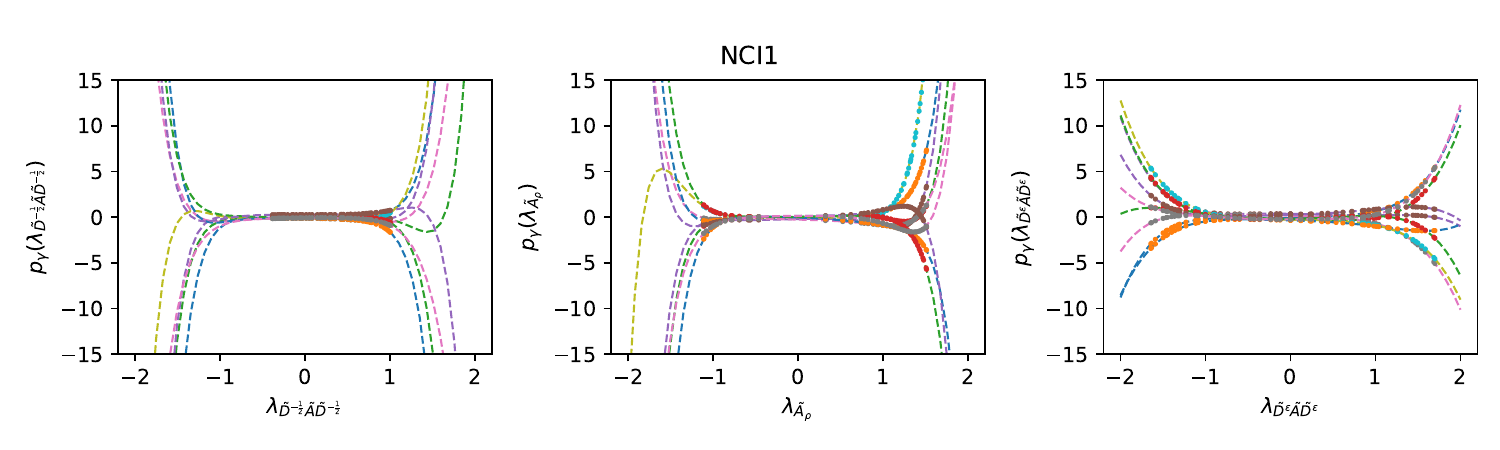}
\end{figure}
\begin{figure}[h]
	\includegraphics[width=0.95\textwidth]{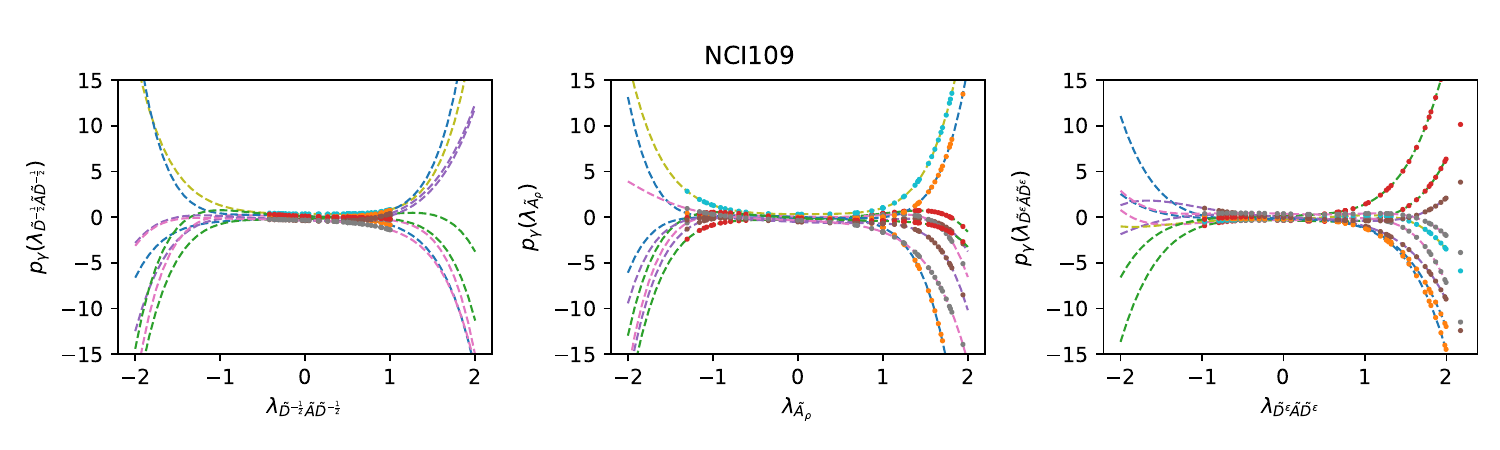}
	\includegraphics[width=0.95\textwidth]{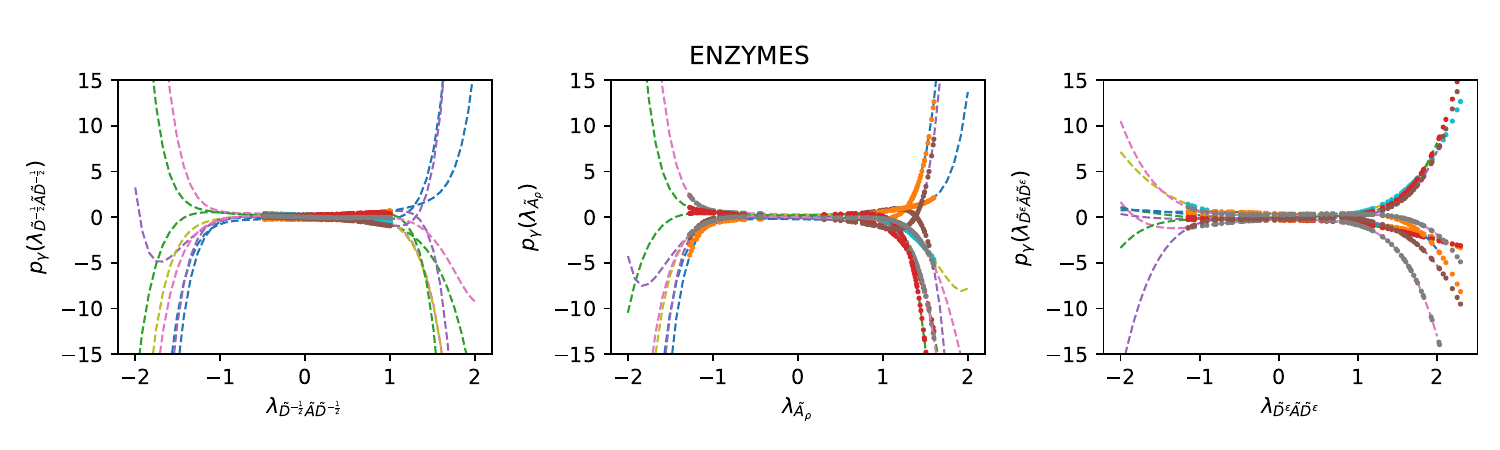}
	\includegraphics[width=0.95\textwidth]{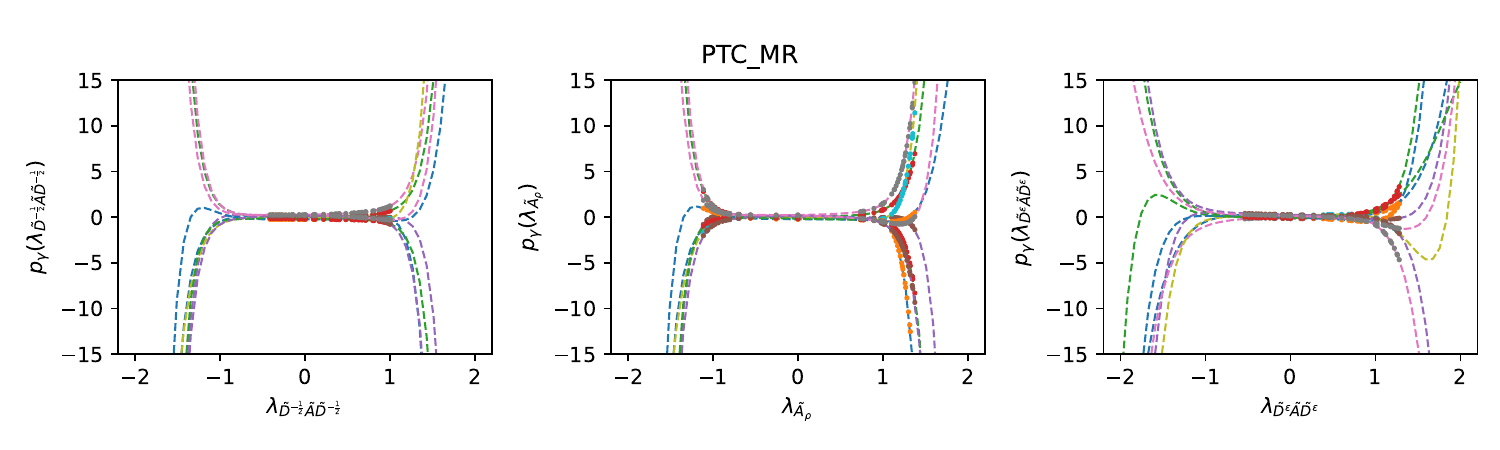}
	\caption{Visualizations of the learned filters on MolPCBA, NCI1, NCI109, ENZYMES and PTC\_MR.}
	\label{fig:filter_visualizations1}
\end{figure}

\FloatBarrier
\section{The Correlation Issue of Deep Models}
\label{ablate_deep}

\begin{figure}[ht]
	\centering
	\includegraphics[width=0.5\textwidth]{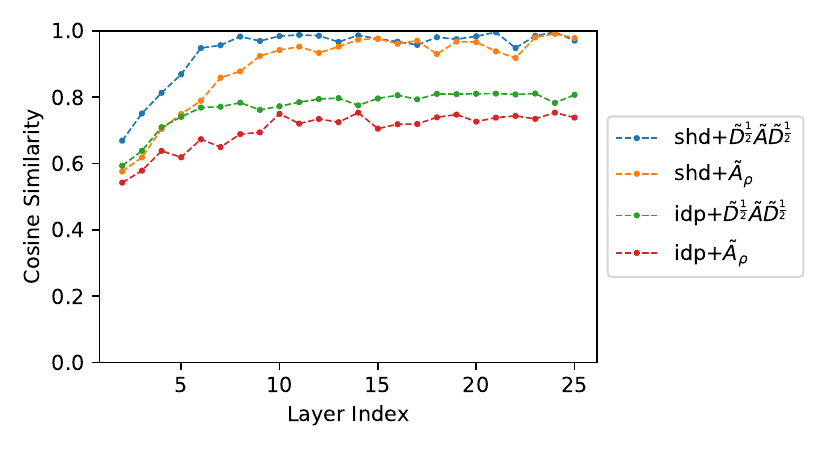}
	\vspace{-10pt}
	\caption{Cosine similarities on ZINC.}
	\label{fig:zinc_cosine}
\end{figure}
We test the absolute value of cosine similarities in different layers for a depth=25 model. For each graph, we compute the mean of all hidden signal pairs. The final visualized results in Fig.\ref{fig:zinc_cosine} are the mean of all graphs within a randomly selected batch. To be consistent with the definition of spectral graph convolution as well as our correlation analysis, the test runs do not utilize edge features of ZINC.

The results show that on both bases, the cosine of the shared filter case converges to 1, while the correlation-free converges to $0.8$ for $\tilde D^{\frac{1}{2}}\tilde A\tilde D^{\frac{1}{2}}$ and $0.7$ for $\tilde A_{\rho}$.
(We also found that it easily leads to a large cosine similarity on ZINC, which is mainly because graphs are small such that $n<<d$, where $n$ is the number of nodes and $d$ is the number of hidden features.)
These results do show that general GNNs suffer from the correlation issue as depth increases, while our correlation-free architecture enjoys a relatively stable performance.

\FloatBarrier
\section{More Results}
\label{more_results}

\begin{figure}[h]
	\centering
	\includegraphics[width=0.95\textwidth]{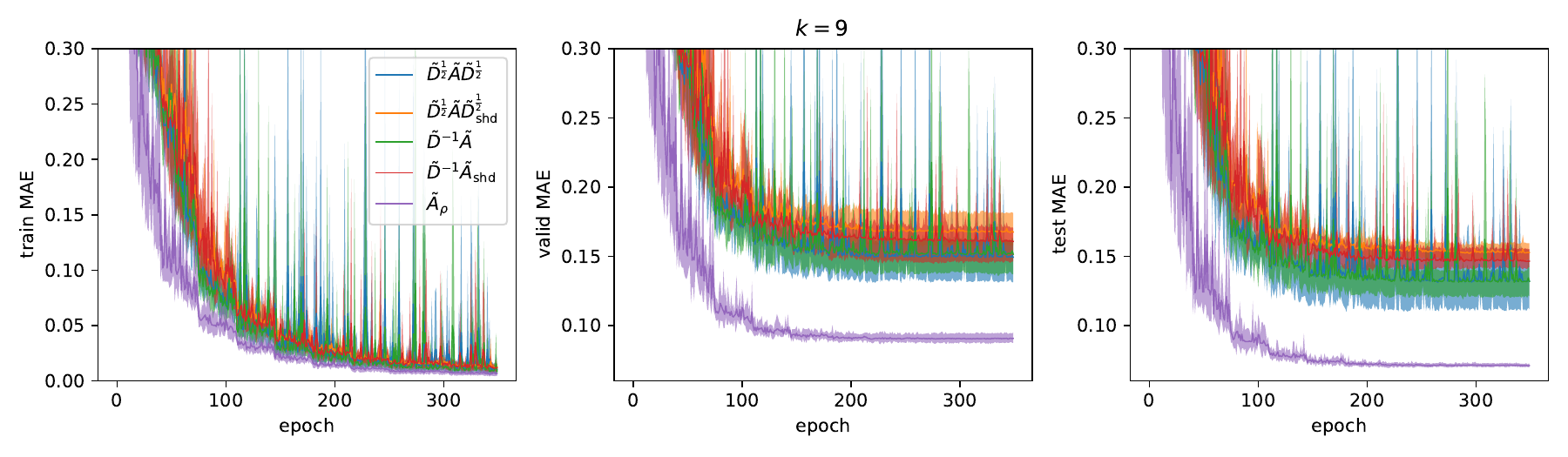}
	\includegraphics[width=0.95\textwidth]{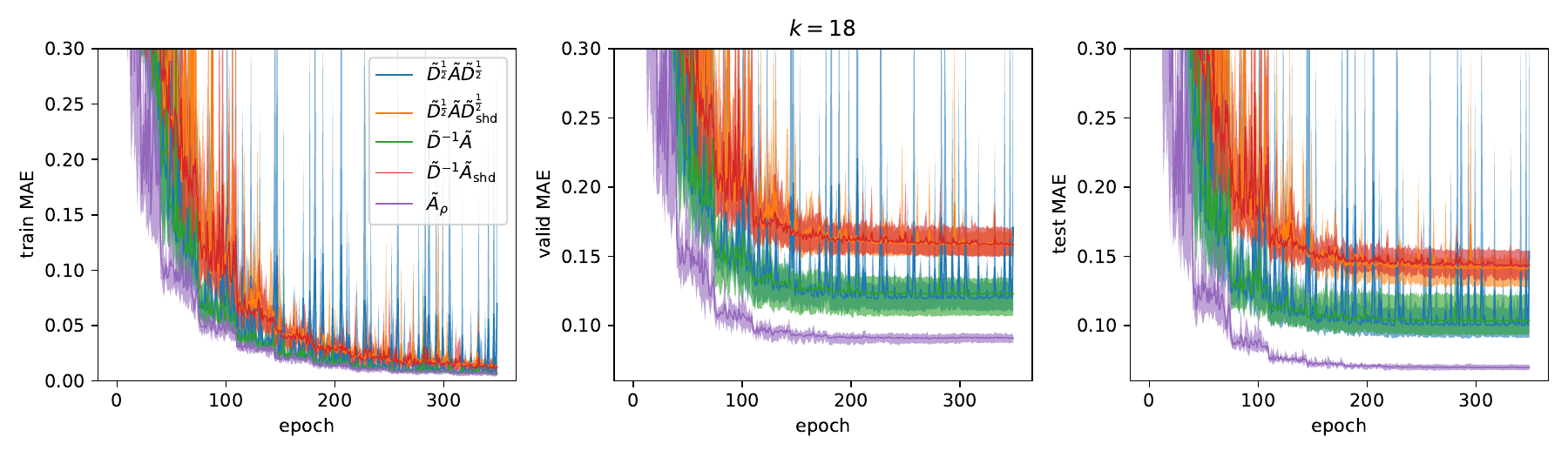}
\end{figure}
\begin{figure}[h]
	\includegraphics[width=0.95\textwidth]{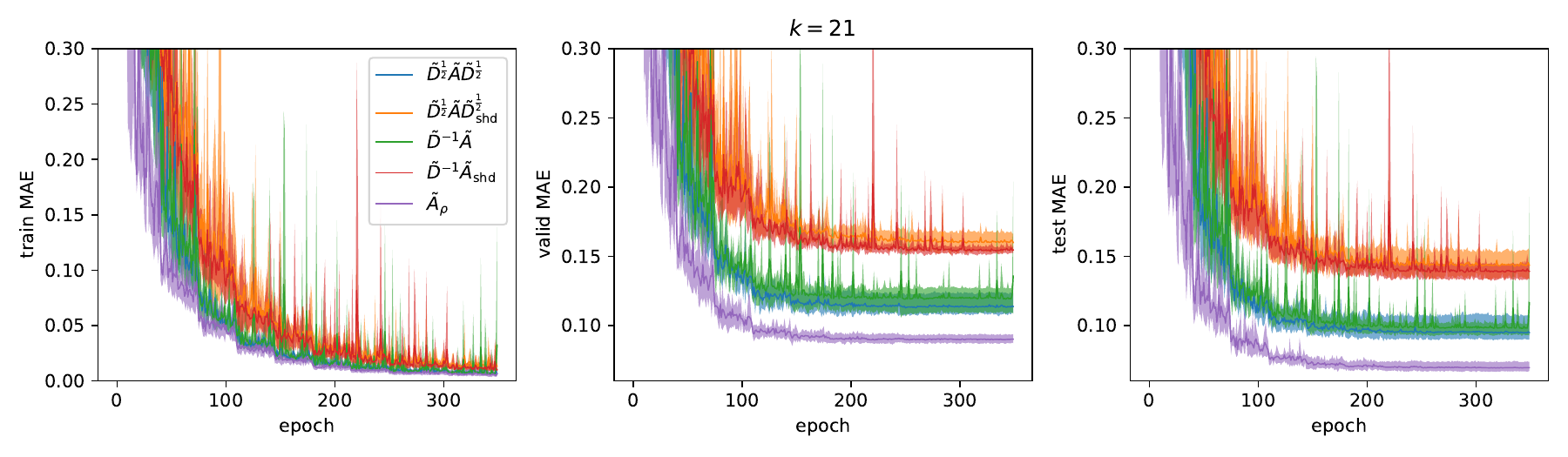}
	\includegraphics[width=0.95\textwidth]{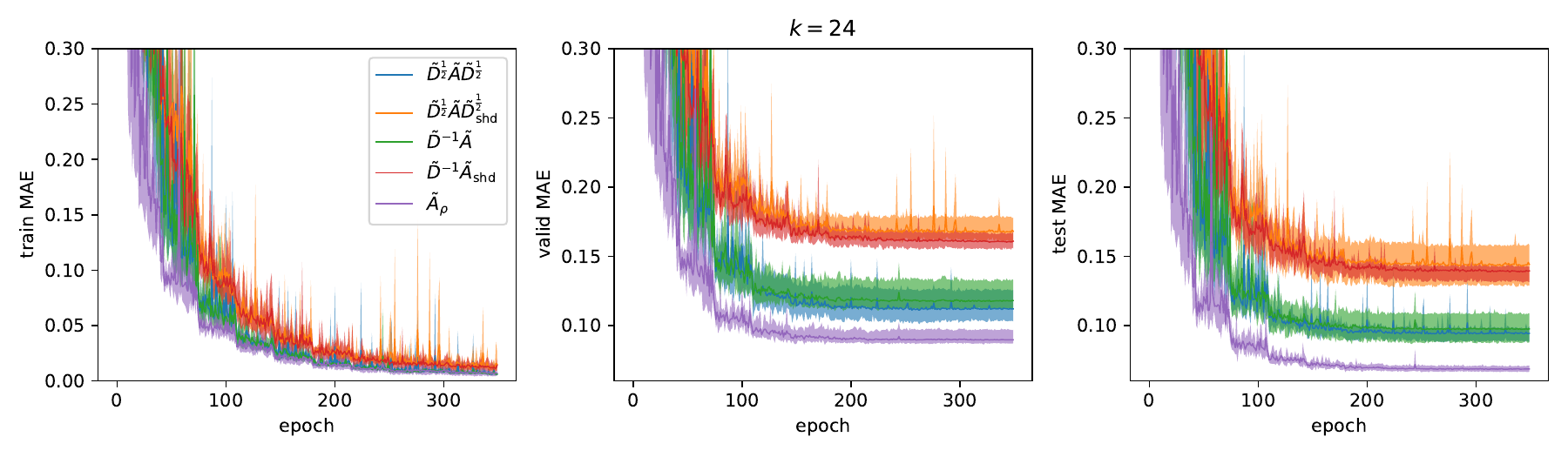}
	\caption{The curves of 5 runs on ZINC with the number of basis $k=9, 18, 21, 24$.}
	\label{fig:more_results1}
\end{figure}

\end{document}